\documentclass[onefignum,onetabnum]{siamart171218}

\usepackage{lipsum}
\usepackage{amsfonts}
\usepackage{graphicx}
\usepackage{epstopdf}
\usepackage{algorithmic}
\ifpdf
  \DeclareGraphicsExtensions{.eps,.pdf,.png,.jpg}
\else
  \DeclareGraphicsExtensions{.eps}
\fi

\newsiamremark{remark}{Remark}
\newsiamremark{hypothesis}{Hypothesis}
\crefname{hypothesis}{Hypothesis}{Hypotheses}
\newsiamthm{claim}{Claim}

\headers{Graph Neural Reaction Diffusion Models}{M. Eliasof, E. Haber, and E. Treister}

\title{Graph Neural Reaction Diffusion Models\thanks{Submitted to the editors June 2023. Accepted in March 2024.
\funding{The research reported in this paper was supported by grant no. 2018209 from the United States - Israel Binational Science Foundation (BSF), Jerusalem, Israel, and in part by the Israeli
Council for Higher Education (CHE) via the Data Science
Research Center, Ben-Gurion University of the Negev, Israel. ME is supported by Kreitman High-tech scholarship.
}}}

\author{Moshe Eliasof\thanks{University of Cambridge, United Kingdom. 
(\email{me532@cam.ac.uk}).}
\and Eldad Haber\thanks{University of British Columbia, Vancouver, Canada. 
  (\email{eldadhaber@gmail.com}).}
\and Eran Treister\thanks{Ben-Gurion University of the Negev, Beer-Sheva, Israel. \email{(erant@cs.bgu.ac.il}).}}

\usepackage{amsopn}

\makeatletter
\newcommand*{\addFileDependency}[1]{
  \typeout{(#1)}
  \@addtofilelist{#1}
  \IfFileExists{#1}{}{\typeout{No file #1.}}
}
\makeatother

\newcommand*{\myexternaldocument}[1]{%
    \externaldocument{#1}%
    \addFileDependency{#1.tex}%
    \addFileDependency{#1.aux}%
}

\ifpdf
\hypersetup{
  pdftitle={Graph Neural Reaction Diffusion Models},
  pdfauthor={M. Eliasof, E. Haber, and E. Treister}
}
\fi

\myexternaldocument{ex_supplement}

\usepackage{microtype}
\usepackage{graphicx}
\usepackage{booktabs} 
\usepackage[normalem]{ulem}
\usepackage{pgfplots}
\usepackage{subfig}
\usepackage{hyperref}
\usepackage{amsmath}

\newcommand{\std}[1]{\pm #1}
\usepackage{multirow}
\newcommand{\bfA}{{\bf A}}

\newcommand{\bfD}{{\bf D}}
\newcommand{\bfE}{{\bf E}}

\newcommand{\bfG}{{\bf G}}

\newcommand{\bfI}{{\bf I}}
\newcommand{\bfJ}{{\bf J}}
\newcommand{\bfK}{{\bf K}}
\newcommand{\bfL}{{\bf L}}

\newcommand{\bfU}{{\bf U}}
\newcommand{\bfV}{{\bf V}}

\newcommand{\bfX}{{\bf X}}
\newcommand{\bfY}{{\bf Y}}

\newcommand{\bfe}{{\bf e}}

\newcommand{\bftheta}{{\bf \theta}}

\newcommand{\bfSigma}{{\bf \Sigma}}

\usepackage{pifont}
\newcommand{\cmark}{\ding{51}}%
\newcommand{\xmark}{\ding{55}}%

\usepackage{tikz}
\usepackage{amsfonts}
\usepackage{graphicx}

\begin{document}

\maketitle

\begin{abstract}
The integration of Graph Neural Networks (GNNs) and Neural Ordinary and Partial Differential Equations has been extensively studied in recent years. GNN 
architectures powered by neural differential equations allow us to reason about their behavior, and develop GNNs with desired properties such as controlled smoothing or energy conservation. 
In this paper we take inspiration from Turing instabilities in a Reaction Diffusion (RD) system of partial differential equations, and propose a novel family of GNNs based on neural RD systems. We \textcolor{black}{demonstrate} that our RDGNN is powerful for the modeling of various data types, from homophilic, to heterophilic, and spatio-temporal datasets. We discuss the theoretical properties of our RDGNN, its implementation, and show that it improves or offers competitive performance to state-of-the-art methods.
\end{abstract}



\section{Introduction}
The emergence of Graph Neural Networks (GNNs)
in recent years made a substantial impact on a wide array of fields, from computer vision and graphics \cite{monti2017geometric, wang2018dynamic,hanocka2019meshcnn} and social network analysis \cite{kipf2016semi, defferrard2016convolutional, chen20simple} to bio-informatics \cite{jumper2021}. 

Recently, GNNs have been linked to ordinary and partial differential equations (ODEs and PDEs, respectively) in a series of works \cite{ zhuang2020ordinary,chamberlain2021grand, eliasof2021pde, rusch2022graph, wang2022acmp, di2022graphGRAFF, gravina2022anti}.
These works propose to view GNN layers as the discretization of ODEs and PDEs, and offer both theoretical and practical advantages. First, such models allow to reason about the behavior of existing GNNs. For instance, as suggested by \cite{chamberlain2021grand}, it is possible to view GCN \cite{kipf2016semi} and GAT \cite{velickovic2018graph} as discretizations of the non-linear heat equation. This observation helps to analyze and understand the oversmoothing phenomenon \cite{nt2019revisiting,oono2020graph,cai2020note}. Second, ODE and PDE\textcolor{black}{-}based GNNs pave the path to the construction and design of novel GNNs that satisfy desired properties, such as energy preservation \cite{eliasof2021pde, rusch2022graph}, attraction and repulsion forces modeling \cite{wang2022acmp, di2022graphGRAFF}, and anti-symmetry \cite{gravina2022anti}.

\begin{figure}[]
    \centering
     \includegraphics[  width=1\textwidth]{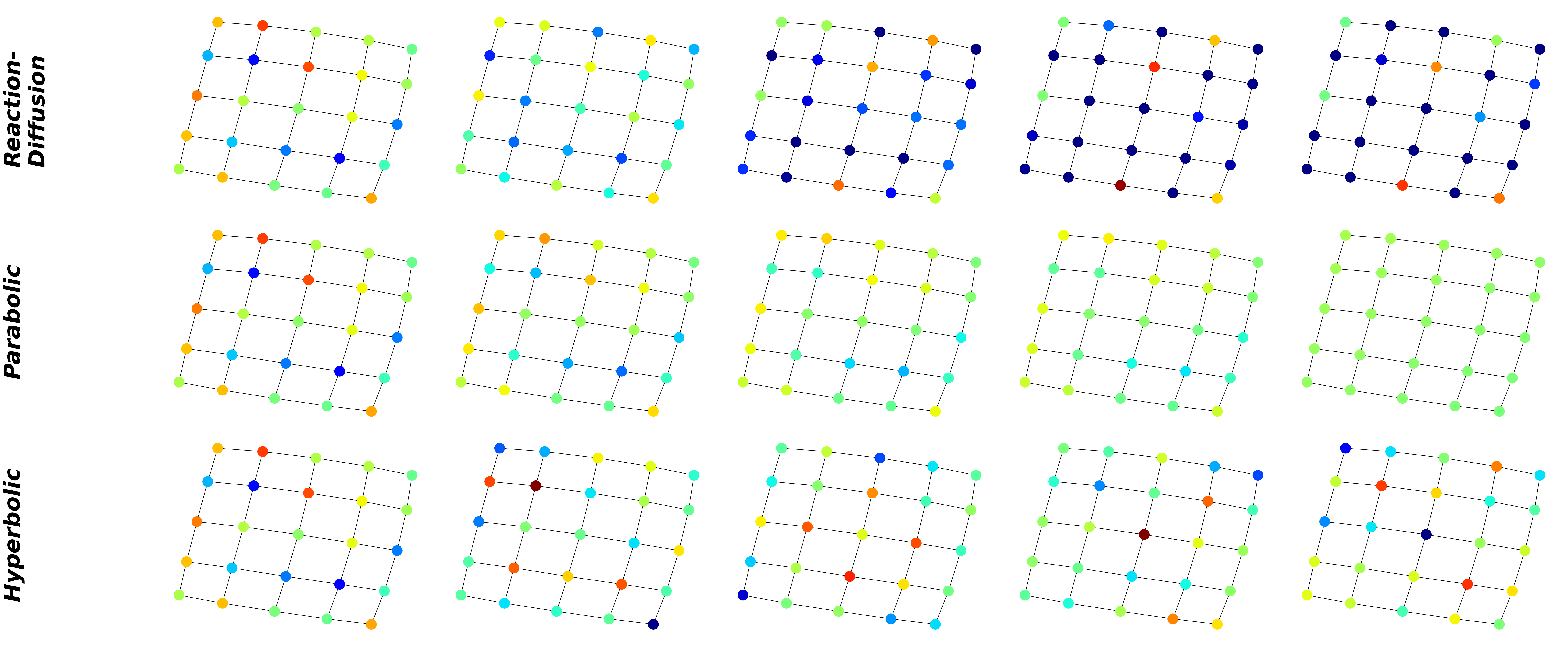}
    \caption{The layer dynamics (portrayed along the x-axis) induced by our RDGNN vs. existing parabolic and hyperbolic dynamics.}
    \label{fig:dynamics}
\end{figure}

In this work, we draw inspiration from instabilities in a reaction diffusion (RD) system \cite{Turing52}, to establish and demonstrate a new class of PDE\textcolor{black}{-}based GNNs that stem from a learnable graph neural RD system. Turing instability within a RD system is one of the fundamental observations in mathematical physics, and it has been used to model different non-linear
phenomena, from the formation of patterns \cite{gierer1972theory}, cell communication  \cite{ratz2012turing}, population dynamics \cite{zincenko2021turing}, to the generation of waves in the cardiac electric system \cite{panfilov2019reaction}, and more \cite{satnoianu2000turing, kondo2010reaction}. Such RD systems can be utilized to explain and model patterns in data, as well as the development of such patterns in time.

\textbf{Contribution.} Building on the seminal work of \cite{Turing52}, we propose RDGNN, a new family of GNN architectures that arise from the discretization of the RD equation on a graph. 
The dedicated construction of our RDGNN and its behavior have different properties compared with existing PDE- and ODE-based GNNs. While parabolic (diffusive) GNNs as in \cite{chamberlain2021grand, eliasof2021pde} lead to piece-wise smooth features, and hyperbolic GNNs as in \cite{eliasof2021pde, rusch2022graph} conserve the feature energy, both dynamics are \emph{globally stable} and therefore do not naturally (i.e., in the continuous domain) generate new feature patterns. \textcolor{black}{Motivated by the profound understanding on classical reaction diffusion systems \cite{Turing52, cottet1993image,kondo2010reaction,chen2017trainable}, and their ability to generate non-smooth patterns that cannot be captured with simple diffusion dynamics as in existing GNNs, we propose to study}  learnable RD system\textcolor{black}{s} \textcolor{black}{that we call} RDGNN, \textcolor{black}{such that} non-smooth patterns using local instabilities \textcolor{black}{can be generated.} 
In \cref{fig:dynamics}, we give examples of the behavior of the RD, parabolic, and hyperbolic dynamics on a graph to illustrate the differences discussed above.
Therefore, we deem that our RDGNN offers greater flexibility in modeling both smooth and non-smooth patterns in data. 
In practice, we show in \cref{sec:experiments} that this qualitative difference translates to improved accuracy on heterophilic node classification datasets, where non-smooth patterns typically appear. Furthermore, in \cref{sec:ablation}, we show that another benefit of our RDGNN is that it does not suffer from oversmoothing in practice.
 Additionally, we show that our RDGNN is effective in modeling and predicting the time-dependent behavior of spatio-temporal systems that reside on graphs in \cref{sec:temporalExperiments}. 

The rest of the paper is structured as follows. In  \cref{sec:related} we discuss related works. In  \cref{sec:method} we describe the continuous RD system, and derive its discretization that operates on graphs. 
We also analyze our RDGNN and demonstrate how instabilities are developed in the process of solving the underlying RD equation on a graph.  

In \cref{sec:experiments} we carry \textcolor{black}{out} numerous experiments on node classification, and spatio-temporal node regression datasets, followed by an ablation study, and 
conclusions in \cref{sec:conclusion}. 

\section{Related Work}
\label{sec:related}

\paragraph{Graph Neural Networks as Diffusion Processes} The seminal work of \cite{scarselli2008graph}, and later additions by \cite{bruna2013spectral, defferrard2016convolutional, kipf2016semi} laid the foundation for current GNNs. 
The common theme of most of these works is their reliance on diffusion processes \cite{nt2019revisiting}. Due to the observed oversmoothing phenomenon \cite{oono2020graph,measuringoversmoothing, cai2020note}, the view and analysis of GNNs as diffusion processes have gained popularity in recent years. Methods like APPNP \cite{klicpera2018combining} propose to use a personalized PageRank \cite{page1999pagerank} algorithm to determine the diffusion of node features, and GDC \cite{gasteiger_diffusion_2019} imposes constraints on the ChebNet \cite{defferrard2016convolutional} architecture to obtain diffusion kernels, showing accuracy improvement. Recently, it was shown in several works that it is also possible to \textcolor{black}{learn} diffusion term from the data \cite{di2022graphGRAFF, eliasof2023improving, chien2021adaptive, luan2022revisiting}, by learning appropriate diffusion coefficients.

\paragraph{Neural ODEs and PDEs}
\label{sec:neuralODEPDE_background}
The interpretation of convolutional neural networks (CNNs) as discretizations of ODEs and PDEs was presented in \cite{RuthottoHaber2018, chen2018deep, zhang2019linked}.   
In the field of GNNs, works like GRAND \cite{chamberlain2021grand}, PDE-GCN\textsubscript{D} \cite{eliasof2021pde} and GRAND++ \cite{thorpe2022grand} propose to view GNN layers as time steps in the integration of the non-linear heat equation, allowing to control the diffusion (smoothing) in the network, to understand oversmoothing. Other works like PDE-GCN\textsubscript{M} \cite{eliasof2021pde} and GraphCON \cite{rusch2022graph} propose to mix diffusion and oscillatory processes to avoid oversmoothing by feature preservation, and A-DGCN \cite{gravina2022anti} proposes an anti-symmetric ODE\textcolor{black}{-}based GNN to alleviate the over-squashing \cite{alon2021oversquashing} phenomenon.

On another note, CNNs and GNNs are also used to solve and accelerate PDE solvers \cite{raissi2018deep, long2018pdenet, li2020multipole, brandstetter2022message}, as well as to generate \cite{sanchez2019hamiltonian} and solve \cite{belbute_peres_cfdgcn_2020} physical simulations, including the solution of the RD equations \cite{li2020reaction}. In this paper, we focus on the view of GNNs as the discretization of PDEs, rather than using GNNs to solve PDEs. 

\paragraph{Reaction Diffusion}
\label{sec:diffusionReactionSystems}
A RD system is a mathematical model that describes the simultaneous spread of a substance or attributes through a medium, and a chemical or biological \textcolor{black}{interactions} taking place within that medium. These systems are used to study a wide range of phenomena, including the spread of chemical and biological agents, the movement of heat and mass in materials, and the flow of information and communication through networks. Recently, the work of \cite{chen2021learning, dodhia2021machine} has been used in the context of Physics informed neural networks to learn the reaction term. This work is
different than ours in the basic assumptions and settings, mainly by the problem domain (continuum vs. a graph) and the training
process. Furthermore, the RD equations have also been applied to various imaging tasks, e.g., \cite{cottet1993image,chen2017trainable, TreisterEtAl2018, haber2019imexnet}.

While applying diffusion on graphs is a well-established concept, RD systems have been less common in GNN literature. In the context of graph RD, works that attempt
to train \emph{parametric} models on surfaces such as the Allen Cahn model have been
proposed in \cite{turk1991generating, bachini2021intrinsic}.
However, these models do not allow for learning a non-parametric (that is, universal) reaction term. Indeed, all works known to us assume a parametric reaction term with a few parameters that need to be determined, typically by trial and error. This makes such models very specific to a particular application, depending on a domain expert to propose an appropriate reaction term. 
Recently, a non-trainable reaction term was  proposed by \cite{wang2022acmp}, demonstrating the benefit of a simple reaction term in GNNs. In this paper we relieve this constraint and propose several trainable reaction schemes while  learning the spatial diffusion coefficients. Closest to ours is the concurrent work of GREAD \cite{choi2022gread} that also learns a graph neural reaction diffusion system. The main differences are in the parameterization of our reaction function (fixed template in GREAD vs. learnable in ours), their integration (by NODE \cite{neuroODE} in GREAD, or by the implicit-explicit approach in our RDGNN), and our experimental evaluation on both static and spatio-temporal datasets.

\section{Neural Graph Reaction Diffusion}
\label{sec:method}
We describe the general outline of a continuous RD system in \cref{sec:RD_cont_systems}. Following that, we present its spatially discrete version in \cref{sec:RD_on_graphs}. We then discuss two time discretization schemes in  \cref{sec:timeDiscretization}, and analyze the behavior of our RDGNN in \cref{sec:stabilityAnalysis}. In \cref{sec:ReactionAndDiffusion_Funcs}, we discuss the reaction and diffusion functions in our RDGNN.

\subsection{Reaction Diffusion Systems} \label{sec:RD_cont_systems}

Reaction diffusion systems allow the modeling of complex phenomena by generating local Turing instabilities.

The RD system of equations has the form of
\begin{subequations}
\label{eq:continuousDiffusionReaction}
\begin{eqnarray}
\label{eq:continuousDiffusionReactiona}
 {\frac {\partial U}{\partial t}} =  \bfSigma{\Delta} {U} + f\left( U, x, \theta 
\right), \ x \in \Omega, \ t \in [0,T]  \\
\label{eq:continuousDiffusionReactionb}
 {U}(0, x) =  U_0( x)
\end{eqnarray}
\end{subequations}

accompanied by some boundary conditions.
Here, $U(t, x) = [u_1(t,  x), \ldots, u_c(t,  x)] : \mathbb{R}^{\textcolor{black}{[0,T] \times \Omega}} \rightarrow \mathbb{R}^{c}$ is a vector of functions where $u_s(t,x) , \ s = 1,\ldots,c$ is a specie (channel), $ \Delta$ is the continuous Laplacian operator applied to
each specie independently. The matrix $\bfSigma = {\rm diag}(\kappa_1, \ldots, \kappa_c) \in \mathbb{R}^{c\times c},  \ \ \kappa_i\ge 0$  is a diagonal matrix of diffusion coefficients, each applied to its corresponding  specie.  $f(U, x, \theta)$ is a reaction function that mixes the different species with some parametric form given parameters $\theta$. The spatial domain $\Omega$ can simply be  $\mathbb{R}^{d}$ or a manifold $\mathcal{M} \subseteq \mathbb{R}^{d}$.

While the range of phenomena that is modeled using Equation \eqref{eq:continuousDiffusionReaction} is large, the difference between these phenomena is governed by two quantities -- the diffusion
coefficients $\bfSigma$, and the non-linear reaction function $f$. Typically, the
function $f(U, x, \theta)$ and the diffusion coefficients $\bfSigma$ are chosen ad-hoc by trial and error \cite{murray2}.

\subsection{Reaction Diffusion on Graphs}
\label{sec:RD_on_graphs}

The RD system in Equation \eqref{eq:continuousDiffusionReaction} is defined in the continuum. We now derive its spatially discrete counterpart to operate on graphs. 
To this end, we assume to have data that resides on a graph $\mathcal{G} = (\mathcal{V}, \mathcal{E})$, where the set of nodes ${\cal V} = \{v_0 \ldots v_{\textcolor{black}{n-1}}\}$ are connected with a set of edges ${\cal E} \subseteq \cal V \times \cal V$ that describe the graph connectivity. 
We also assume that each node $v_i$ is accompanied with a data pair, $\{\bfX_i, \bfY_i\}$, where $\bfX_i \in \mathbb{R}^{c_{in}}$ are the input node features, and $\bfY_i \in \mathbb{R}^{c_{out}}$ is the ground-truth data. Let us denote the adjacency matrix $\bfA \in \mathbb{R}^{n \times n}$, where $\bfA_{ij} = 1$ if there exists an edge $(i,j) \in {\cal E}$ and 0 otherwise. We also define the diagonal degree matrix $\bfD \in \mathbb{R}^{n \times n}$, where $\bfD_{ii}\textcolor{black}{=d_i}$ is the degree of the $i$-th node. The graph Laplacian is then given by $\bfL=\bfD-\bfA$. 
Let us also denote the adjacency and degree matrices with added self-loops by $\tilde \bfA$ and $\tilde \bfD$, respectively, and the corresponding graph Laplacian by $\tilde{\bfL} = \tilde{\bfD} - \tilde{\bfA}$.

To define the spatially discrete counterpart of Equation \eqref{eq:continuousDiffusionReaction}, we replace the continuous Laplacian operator $\Delta$ with the discrete symmetric normalized
graph Laplacian 
\begin{equation}\label{eq:hatL}
\hat{\bfL} = \tilde{\bfD}^{-\frac{1}{2}}\tilde{\bfL}\tilde{\bfD}^{-\frac{1}{2}},
\end{equation}
that operates similarly to the continuous Laplacian on data that reside on a graph. Note that due to this discretization, in what follows, $\bfSigma$ will be multiplied from the right.

The spatially discrete analog to the RD
system from Equation \eqref{eq:continuousDiffusionReaction} is therefore given by
\begin{subequations}
\label{eq:discDR}
\begin{eqnarray}
\label{eq:discDRa}
{\frac {d\bfU(t)}{d t}} &=&  -\hat\bfL \bfU(t) \bfSigma(t) + f(\bfU(t), \bfX, t; \bftheta(t)) \\
\label{eq:discDRb}
\bfU(0) &=& \bfU_0 = f_0(\bfX\textcolor{black}{; } \bftheta_0) 
\end{eqnarray}
\end{subequations}
Here, $\bfU(t)\in\mathbb{R}^{n \times c}$ is a hidden state that depends on the initial state $\bfU_0$, and the input node features $\bfX\in \mathbb{R}^{n \times c_{in}}$. \textcolor{black}{The initial state is obtained by embedding the input node features using a single layer MLP with learnable weights $\theta_{0}$ denoted by $f_{0}$, as follows: $\bfU_{0} = f_{0}(\bfX; \theta_{0})$. Also,}
$\bfSigma(t) \textcolor{black}{\geq} 0$ is a diagonal matrix with non-negative learnable diffusion coefficients on its diagonal. The reaction term  $f$ from Equation \eqref{eq:discDRa} is a learnable point-wise non-linear function that depends on the input node features 
$\bfX$ and the current hidden state $\bfU(t)$. 
Note that in classical PDE theory, and in particular in the definition of a RD system, the Laplacian is a negative operator, while in graph theory it is positive. We therefore added the negative sign to $\hat\bfL$ in Equation \eqref{eq:discDRa} compared to Equation \eqref{eq:continuousDiffusionReactiona}.

Our model is more general than the standard RD system in Equation \eqref{eq:continuousDiffusionReaction}, because it allows
both the reaction $f$ and the diffusion coefficients $\bfSigma$ \textcolor{black}{to be} time-dependent. We add this flexibility for scenarios when  diffusion policies or the reaction can be time-varying, 
for example, in \textcolor{black}{periodic phenomena that depend} on the month of the year such as the spread of disease\textcolor{black}{s}.
We note that typically, $c$, the number of channels of $\bfU(T) \in \mathbb{R}^{n \times c}$, is different than $c_{out}$, the number of channels of the target output $\bfY \in \mathbb{R}^{n \times c_{out}}$. Therefore the output of \textcolor{black}{the} neural network, $\hat{\bfY}$, is given by 
\begin{eqnarray}
\label{yFromu}
\hat{\bfY} = g(\bfU(T)\textcolor{black}{; } \bftheta_{\rm out}) \in \mathbb{R}^{n \times c_{out}}
\end{eqnarray}
where  $g$ is a simple classifier (a single $1\times1 $ convolution), parameterized by learnable weights $\bftheta_{\rm out}$.

\subsection{Time Integration}
\label{sec:timeDiscretization}
While the RD system in Equation \eqref{eq:discDRa} is discretized in space, it is still continuous in time. In this paper we interpret the time steps as neural network layers, and, in what follows, we replace the time notation $t$ with layer index $l$, with a positive time step size $h$, such that $t=h l$.
We now discuss two popular time discretization schemes.
\paragraph{Explicit Integration}
While   explicit forward Euler discretization is very restrictive due to \textcolor{black}{numerical} stability,
it is still one of the most used for deep neural networks and is known as ResNet \cite{he2016deep}, which we write for completeness as:
\begin{equation}
    \label{eq:forwardEulerDiscretization}
    \bfU_{l+1} = \bfU_l - h \left(\hat\bfL \bfU_{l} \bfSigma_{l} - f(\bfU_l, t_l; \bftheta_l)\right). 
\end{equation}
However, such an explicit scheme requires using a small step size $h>0$, or equivalently, many layers, as it is \textcolor{black}{numerically} marginally stable \cite{ascherBook}. 

\paragraph{Implicit-Explicit (IMEX) Integration}
To lift the restrictions of the forward Euler method, we harness the IMEX scheme
that is commonly used for RD equations, and allows us to consider a larger $h$ step size \cite{ruuthphd, coats2000note}. 

The IMEX integration scheme for our RDGNN reads 
\begin{eqnarray}
\label{imexdisc}
{\frac {\bfU_{l+1} - \bfU_l}{h}} =  -\hat\bfL \bfU_{l+1} \bfSigma_{\textcolor{black}{l}} + f(\bfU_l,  t_l; \bftheta_l), 
\end{eqnarray}
leading to a discretized ODE, solved as follows:
\begin{subequations}
\label{eq:drSolutionWithImex}
\begin{eqnarray}
\label{eq:dr2}
&& \bfV_l = \bfU_l + h f(\bfU_l,  t_l; \bftheta_l)   \\
\label{eq:dr3}
&& \bfU_{l+1} = {\rm mat} ((\bfI + h  \bfSigma_{l} \otimes \hat\bfL)^{-1} {\rm vec} ({\bfV}_l))
\end{eqnarray}
\end{subequations}
where $l=1, \ldots, L$. 

The ${\rm vec}(\cdot)$ operation  flattens the input tensor, and ${\rm mat}(\cdot)$ is the inverse reshape of the $\rm vec(\cdot)$ operation. $\otimes$ is the Kronecker product.
Note, that the IMEX iteration of the RD equations has two main steps. Step \eqref{eq:dr2} performs explicit integration of the reaction function $f$, resulting in a skip-connection term.
This step can be thought of as a ResNet \cite{he2016deep} layer that is stepped forward on each node in $\mathcal{V}$ independently.  
Next, step \eqref{eq:dr3} describes the implicit step that
requires the solution of a linear system.
The step propagates the information between the graph nodes $\mathcal{V}$.
Note that the term $(\bfI + h  \bfSigma_{l} \otimes \hat\bfL)^{-1}$ is dense, and it therefore allows to potentially propagate and share information between each node-pair $(v_i, v_j) \in \mathcal{V} \times \mathcal{V}$ in the graph. Also, note that because $\bfSigma_l$ is a non-negative diagonal matrix and $\hat\bfL$ is symmetric positive semi-definite, then $\bfI + h  \bfSigma_{l} \otimes \hat\bfL$ is symmetric positive definite, and invertible. 

To invert the matrix in step \eqref{eq:dr3} one may use a direct or an iterative method. 
Direct methods like the Cholesky factorization are suitable for small graphs (up to a few thousand nodes). 
However, here we use the Conjugate Gradient (CG) iterative method \cite{ascherBook}, as we typically observe fast convergence (5-10 iterations) for a relative error that is smaller than $10^{-2}$. We refer the reader to  \cref{app:complexity} for a discussion about backpropagation through the implicit step.

The use of IMEX methods was previously utilized in CNNs \cite{haber2019imexnet} and GNNs \cite{chamberlain2021grand}. To validate the advantages of IMEX over a forward Euler integration in our RDGNN, we provide experimental results on several datasets that favor the IMEX integration scheme, in our ablation study in Section \ref{sec:ablation}.

\subsection{Local Stability Analysis}
\label{sec:stabilityAnalysis}

At the core of RD systems lie the local instabilities, for a short time, that promote pattern growth. To demonstrate how such instabilities can be achieved in our RDGNN, we state the following theorem.

\begin{theorem}
\textcolor{black}{Consider} the discrete RD equation defined by Equation \eqref{eq:discDRa}, where $\hat\bfL$ is the symmetrically normalized graph Laplacian, $\bfSigma = {\rm diag}(\kappa_1, \ldots, \kappa_c)$ with $\kappa_j\textcolor{black}{\geq}0$ are the diffusion coefficients, and $f(\bfU,\theta)$ is a non-linear reaction function. If the matrix 
\begin{equation}\label{eq:G}
\bfG = -\bfSigma \otimes \hat\bfL  + \bfJ
\end{equation} where $\bfJ=\nabla_\bfU f$, contains positive eigenvalues, then the RD system \eqref{eq:discDRa} is unstable for the eigenvectors that are associated with those eigenvalues.
\end{theorem}
\begin{proof}
 Because Equation \eqref{eq:discDRa} is non-linear, we turn to understand its qualitative behavior through a linear
analysis, a common tool for dynamical systems (see \cite{EvansPDE, Turing52, turing1990chemical} and references within).

Examining the properties of Equation~\eqref{eq:discDR} around some point $\bfU_0$ that satisfies the discrete ODE from Equation \eqref{eq:discDR}, letting $\bfU(t) = \bfU_0(t) + \delta \bfU(t)$ we obtain
\begin{equation}
\label{eq:analysisLinearization2}
\frac{d(\bfU_0 + \delta \bfU)}{d t} = -\hat\bfL (\bfU_0 + \delta \bfU) \bfSigma + f(\bfU_0 + \delta \bfU, \textcolor{black}{\bfX}, \theta)
\end{equation}
To study the behavior of the equation, let us rewrite the equation in a vectorized form 
\begin{equation}
\small
\label{eq:analysisLinearization3}
\frac{d({\rm vec}(\bfU_0 + \delta\bfU))}{d t} = -(\bfSigma\otimes \hat\bfL)
{\rm vec}(\bfU_0 +  \delta\bfU) + f({\rm vec}(\bfU_0 +  \delta\bfU), \bfX, \theta)
 \end{equation}
where  $\otimes$ is the Kronecker product, hence the matrix $\bfSigma \otimes \hat\bfL \in \mathbb{R}^{nc \times nc}$.  
Using a first-order linearization for the reaction term $f$ we obtain:
\begin{equation}
    \label{eq:reactionLinearization}
    f({\rm vec}(\bfU_0 + \delta \bfU), \bfX, \theta) \approx
f({\rm vec}(\bfU_0), \bfX, \theta)  +  \bfJ  {\rm vec}( \delta\bfU).
\end{equation}

where $\bfJ$ is Jacobian of $f$ with respect to ${\rm vec} (\bfU)$, i.e., $\bfJ = \nabla_{\bfU} f$.

Altogether, we obtain a local linear ODE with respect to $\delta \bfU$
\begin{equation}
    \label{eq:deltaU_ode}
    {\frac{d {\rm vec}(\delta \bfU)}{dt}} =    
-\bfSigma\otimes \hat\bfL  {\rm vec}(\delta\bfU)  +  \bfJ {\rm vec}(\delta \bfU) = \bfG {\rm vec}(\delta \bfU)
\end{equation}
where $\bfG$ is the matrix defined in Equation \eqref{eq:G}. This ODE describes the local behavior of our RDGNN model from Equation \eqref{eq:discDR}. Therefore, we obtain that the matrix of interest is $\bfG$ in Equation \eqref{eq:G}, and if this matrix has positive eigenvalues, the associated eigenvectors will induce instability.
\end{proof}

\begin{remark}
Note that the term $-\bfSigma \otimes \hat\bfL$ is always negative semi-definite as each of its eigenvalues is minus the product between an eigenvalue of $\bfSigma$ (positive) and an eigenvalue of $\hat\bfL$ (non-negative Laplacian). Therefore, this term in \cref{eq:G} alone induces stability due to its monotone nature \cite{ap}. Nonetheless, there are two possible instabilities that can be observed or developed,   due to the addition of $\bfJ$, the Jacobian of the reaction function $f$.
\end{remark}
First, if the real part of the eigenvalues of $\bfJ$ are positive, it is possible that the matrix $\bfG$ will have eigenvalues with
positive real parts that induce local instability. Second, 
an interesting observation is that even if
both $-\bfSigma \otimes \hat\bfL$ and $\bfJ$ have  \textcolor{black}{negative real part} eigenvalues, and each is stable and monotone on its own---adding the Laplacian eigenvalues weighted by the diffusion coefficients $\bfSigma$, to $\bfJ$, \emph{can} result in a matrix with some real positive eigenvalues, $\bfG$ from \cref{eq:G}, that describes an unstable process. 
In such a case, the Turing instabilities occur \cite{satnoianu2000turing}. 

We provide a simple $2\times2$ example of such a case. Consider the matrices
$$ 
A = \begin{bmatrix}
   -1.0 &   0.1 \\
    -1.0  &  0.0
\end{bmatrix}, \quad B = \begin{bmatrix}
   -1.0 &   -1.0 \\
    0.1  &  0.0
\end{bmatrix}.
$$

It is easy to verify that both matrices have  negative eigenvalues. However, their sum $A + B$ yields one positive and one negative \textcolor{black}{eigenvalue}. This positive eigenvalue implies that the eigenvector that is associated with this eigenvalue will become unstable, which is desired in the context of pattern formation using RD systems.

We refer the interested reader to \cite{satnoianu2000turing} for further analysis of the generation of such instabilities. It is important to note, that these systems are typically analyzed for fixed RD systems, while here, we allow our RDGNN to learn the reaction and diffusion parameters from the data in a task-driven fashion. Consequently, either instability of the first kind (positive real part of the eigenvalues of $\bfJ$) or Turing instability can be generated if required by the data and task at hand. 
Generating such instabilities requires a combination of the diffusion coefficients $\bfSigma$ and the non-linear function $f$ and therefore, in our model, we assume both can be learned from the data, which we also discuss later in Section \ref{sec:method}.

\textcolor{black}{\textbf{Mitigation of oversmooting in RDGNNs.} In this section we theoretically show how the RDGNN is able to mitigate oversmoothing. For simplicity, we consider a single-channel version of RDGNN with a linear reaction term and step size $h=1$. In this case, an RDGNN layer reads: 
\begin{equation} \label{eq:1DRDGNN}
\bfU^{(l+1)} = \bfU^{(l)} - \kappa\tilde\bfL\bfU^{(l)} + k_{U}\bfU^{(l)},
\end{equation}
where $\kappa\geq 0$ is a learnable scalar that is a diffusion coefficient, and is the counterpart of the multi-channel RDGNN diffusivity coefficients $\bfSigma$. Also, $k_U$ is a learnable weight representing the $1\times 1$ convolution in one dimension, which is the counterpart of the multi-channel terms in the reaction function $f$.}

\textcolor{black}{The typical oversmoothing analyses of GCN \cite{nt2019revisiting, cai2020note}, can be realized when setting $\kappa = 1$ and $k_{U} = 0$. These settings yield the following node feature update rule:
\begin{equation}
\label{eq:gradFlow}
    \bfU^{(l+1)} = \bfU^{(l)} - \tilde\bfL\bfU^{(l)}.
\end{equation}
As has been shown in \cite{di2022graphGRAFF}, Equation \eqref{eq:gradFlow} can be interpreted as the gradient flow of the Dirichlet energy:
\begin{equation}\label{eq:Dirichlet}
E{(\bfU^{(l)})} = \sum_{i \in \mathcal{V}} \sum_{(i,j) \in \mathcal{E}} \frac{1}{2} \textstyle{\left\|\frac{\bfU^{(l)}_i}{\sqrt{(1+d_i)}} - \frac{\bfU^{(l)}_j}{\sqrt{(1+d_j)}} \right\|_2^2} = ||\nabla \bfD^{1/2} \bfU^{(l)}||_2^2.
\end{equation}
The Dirichlet energy is a common metric to measure smoothness in GNNs \cite{nt2019revisiting,cai2020note,rusch2022graph}, and its decrease to zero as we add more layers corresponds to the oversmoothing issue in GNNs. Algebraically, the update rule presented in Equation \eqref{eq:gradFlow} can be described by the application of the operator $\bfI-\tilde{\bfL}$ to the current node features $\bfU^{(l)}$. Note, that the eigenvectors corresponding to the largest eigenvalues of the operator $\bfI - \tilde\bfL$ are spatially smooth. Therefore, repeated applications of this matrix will converge to smooth feature maps, and applying too many such layers will eventually oversmooth. 
In \cite{di2022graphGRAFF, eliasof2023improving}, it is shown that including a learnable diffusion weight $\kappa$ leads to the algebraic operator $\bfI-\kappa\tilde{\bfL}$, and can control the smoothness and prevent oversmooting. Here, this option is also viable due to the learning of $\kappa$. Additionally, in RDGNN the reaction term can also aid in preventing oversmoothing, by shifting the eigenvalues of the effective algebraic operator in \eqref{eq:1DRDGNN}, that reads $\bfI-\kappa\tilde\bfL + k_U\bfI$. For example, placing $k_U = -1$, will result in 
\begin{equation}
\bfU^{(l+1)} = - \kappa\tilde\bfL\bfU^{(l)},    
\end{equation}
for which the eigenvector with the largest eigenvalue in magnitude corresponds to a high-frequency feature map. Furthermore, placing $k_U = -2$ will result in a gradient ascent iteration (up to a minus sign) on the Dirichlet energy term in \eqref{eq:Dirichlet}. Thus, in those two possible realizations offered by our RDGNNs, the Dirichlet energy will increase and the corresponding node features will have amplified high-frequency components compared to the input $\bfU^{(l)}$.}

\textcolor{black}{To summarize, the RDGNN architecture has the flexibility to \emph{enlarge and reduce} the Dirichlet energy, and the weights for controlling the energy are determined using the learning problem on each given dataset and task.}

\subsection{The Reaction and Diffusion Functions}
\label{sec:ReactionAndDiffusion_Funcs}

\paragraph{The Reaction Function}
Our RDGNN employs two types of neural reaction layers, which we now discuss. For brevity, we omit the layer notation $l$, and discuss per-layer parameterization or cross-layer shared parameters in \cref{sec:experiments}.

The first reaction term is an additive multi-layer perceptron (MLP) that considers both the initial node embedding $\bfU_{0}$ and the current node features $\bfU$.
\begin{equation}
\small
    \label{eq:simpleAdditiveReaction}
        f_{\rm add}(\bfU, \bfX, t) = \sigma \left(\bfU \bfK_{U}  + \bfU_0(\bfX)\bfK_{U_0}  +  \phi(t \textcolor{black}{\mathbf{1}{\bfe_{f}}^\top}) \bfK_t  \right), 
\end{equation}

where 
$\sigma$ is a non-linear activation function (ReLU in our experiments), and $\bfU_0(\bfX)$ is defined by \cref{eq:discDRb}. $\bfK_{U}$, $\bfK_{U_{0}}$, and $\bfK_t$  are $c\times c$ learnable weights of a $1 \times 1$ convolution.
The learned time embedding vector $\bfe_{f} \in \mathbb{R}^{c \times \textcolor{black}{1}}$ is multiplied by $t$ and is placed in a periodic, \textcolor{black}{element-wise} sine activation
function $\phi(\cdot)$. $\mathbf{1} \in \mathbb{R}^{n  \textcolor{black} {\times 1}}$ is a vector \textcolor{black}{whose entries are ones. The multiplication $\mathbf{1}{\bfe_{f}}^{\top}$} repeats the time embedding to all the nodes in $\mathcal{V}$ \textcolor{black}{resulting in a matrix $n \times c$}. The role \textcolor{black}{of} the time embedding
is to expose possible periodicity in the data, that often appears in temporal datasets. 

While additive layers work well in our experiments, they often fail to model multiplicative functions \cite{Jayakumar2020Multiplicative} that may appear in RD systems. For instance, the SIR \cite{weiss2013sir} model of the spread of infectious diseases and predator-prey behavior contains such reaction functions. Thus, we also consider a multiplicative reaction function of the form
\begin{equation}
\label{eq:multiplicativeLayer}
f_{\rm mul}(\bfU, \bfX, t) = f_{\rm add_1}(\bfU, \bfX, t)
\odot 
f_{\rm add_2}(\bfU, \bfX, t),
\end{equation}
where $\odot$ is the element-wise Hadamard product. 

To allow our RDGNN to express both additive and multiplicative reactions, we define our reaction function as:
\begin{equation}
    \label{eq:totalReaction}
    f(\bfU, \bfX, t) = f_{\rm{add_{1}}}(\bfU, \bfX, t)  + f_{\rm{mul}}(\bfU, \bfX, t). 
\end{equation}
\textcolor{black}{Note, that the weights of the term $f_{\rm{add_{1}}}$ are shared between the additive and multiplicative terms.}

\paragraph{The Diffusion Term} In this work we use the symmetric normalized Laplacian matrix, defined in Equation \eqref{eq:hatL}, multiplied by the diffusion coefficients $\bfSigma$. Other diffusion terms, such as a learnable diffusion operator using a symmetrized graph attention operator \cite{velickovic2018graph}, may be considered and are left for future work. 
To obtain a valid non-negative diagonal matrix $\bfSigma$, we parameterize the learnable diagonal matrix $\hat{\bfSigma}$ by
\begin{equation}
\label{eq:diffusion}
\bfSigma(t) = \exp\left(-{\rm{ReLU}}(\hat{\bfSigma} + \phi(t\bfE_{\bfSigma})) \right).
\end{equation} 
The exponent is often used to model a diffusion term 
for inverse problems of diffusion coefficients \cite{SchillingsStuart2017}. The matrix $\bfE_{\bfSigma} \in \mathbb{R}^{c \textcolor{black}{\times c}}$ is a time embedding \emph{diagonal} matrix that allows for a change in the diffusion coefficient over time.

\section{Experiments}
\label{sec:experiments}

We demonstrate our RDGNN on two types of tasks -  node classification and spatio-temporal node regression. The exact architecture specification is provided in Appendix \ref{appendix:architectures}. In all our experiments, we use the Adam \cite{kingma2014adam} optimizer, and perform a grid search to choose the hyper-parameters. More details about the hyper-parameters are given in Appendix \ref{appendix:hyperparams}. Our code is implemented using PyTorch \cite{pytorch} and PyTorch-Geometric \cite{pyg2019}, trained on an Nvidia RTX3090 with 24GB of memory.
The details and statistics of the datasets used in our experiments are provided in Appendix \ref{appendix:datasets}. \textcolor{black}{In Appendix \ref{app:complexity}, we provide measured training and inference runtimes of RDGNN and compare it with other methods. Our results show that while RDGNN requires added computations due to the added components, it also offers improved performance.}

We propose three variants of our RDGNN, as follows:
\begin{itemize}
    \item RDGNN-S. In this architecture the parameters $\bfK_{U} , \ \bfK_{U_{0}} , \ \bfK_t$, $\hat{\bfSigma}$,  $\bfe_f$, and $\bfE_{\bfSigma}$ are shared among all the layers
. This option is aligned with the idea of the final state being an approximate steady state of the classic RD process \cite{Turing52}.
    \item RDGNN-I. Here we follow a similar approach to typical neural networks, where different weights are learned for each layer. From an RD perspective, this can be interpreted as an RD unrolling iteration \cite{mardani2018neural}.
    \item RDGNN-T. A time-dependent RDGNN that stems from our RDGNN-S. In addition to RDGNN-S, it offers a time-dependent feature embedding $\bfe_f$, and $\bfE_{\bfSigma}$ from Equations \eqref{eq:simpleAdditiveReaction} and  \eqref{eq:diffusion}, respectively. 
\end{itemize}
Note, that in the case of RDGNN-S and RDGNN-I, the time embedding terms  \textcolor{black}{$\bfe_{f}$ and $\bfE_{\Sigma}$} are omitted. Further discussion is provided in \cref{appendix:architectures}.

\subsection{Node Classification}
\label{sec:nodeExperiments}
\paragraph{Homophilic data}
We experiment with the homophilic Cora \cite{mccallum2000automating}, Citeseer \cite{sen2008collective}, and Pubmed \cite{namata2012query} graph datasets. We use the 10 publicly available splits from \cite{Pei2020Geom-GCN:} with train/validation/test split ratio of $48 \%, 32\%, 20\%$ respectively, and report their average accuracy in \cref{table:homophilic_fully_std}. 
In an effort to establish a solid baseline, we consider multiple recent methods, such as GCN \cite{kipf2016semi}, GAT \cite{velickovic2018graph}, Geom-GCN \cite{Pei2020Geom-GCN:}, APPNP \cite{klicpera2018combining}, JKNet \cite{jknet}, MixHop \cite{abu2019mixhop},  WRGAT\cite{Suresh2021BreakingTL}, GCNII \cite{chen20simple}, PDE-GCN \cite{eliasof2021pde}, NSD \cite{bodnar2022neuralsheaf}, H2GCN \cite{zhu2020beyondhomophily_h2gcn}, GGCN \cite{yan2021two}, C\&S \cite{huang2020combining}, DMP \cite{yang2021diverse}, LINKX \cite{lim2021large}, and  ACMII-GCN++ \cite{luan2022revisiting}, and GREAD \cite{choi2022gread}. We see that our RDGNN-I outperforms all methods on the Cora and Pubmed datasets, and achieves close (0.14\%  accuracy difference) to the best performing PDE-GCN\textsubscript{M} on Citeseer.

\begin{table}[h]
  \caption{Node classification accuracy ($ \%$) on \emph{homophilic} datasets. $\dagger$ denotes the maximal accuracy of several proposed variants.}
  \label{table:homophilic_fully_std}   
  \center{
  \begin{tabular}{cccc}
    \toprule
    Method & Cora & Citeseer & Pubmed \\
    Homophily & 0.81 & 0.80 & 0.74 \\
    \midrule
    GCN  & 85.77 $\pm$ 1.27 & 73.68 $\pm$ 1.36 & 88.13 $\pm$ 0.50  \\
    GAT & 86.37 $\pm$ 0.48 & 74.32 $\pm$ 1.23 & 87.62 $\pm$ 1.10 \\
    GCNII\textsuperscript{$\dagger$} & 88.49 $\pm$ 1.25  & 77.13 $\pm$ 1.48 & 90.30 $\pm$ 0.43  \\
    Geom-GCN\textsuperscript{$\dagger$} & 85.27 $\pm$ 1.57  & 77.99 $\pm$ 1.15  & 90.05 $\pm$ 0.47\\
    MixHop  & 87.61 $\pm$ 2.03 & 76.26 $\pm$ 2.95 & 85.31 $\pm$ 2.29 \\
    WRGAT & 88.20 $\pm$ 2.26 & 76.81 $\pm$ 1.89 & 88.52 $\pm$ 0.92 \\
    NSD\textsuperscript{$\dagger$}  & 87.14 $\pm$ 1.13  & 77.14 $\pm$ 1.57   & 89.49  $\pm$ 0.40\\
    GGCN  & 87.95 $\pm$ 1.05  & 77.14  $\pm$ 1.45  & 89.15 $\pm$ 0.37 \\
    H2GCN  & 87.87 $\pm$ 1.20  & 77.11 $\pm$ 1.57  & 89.49 $\pm$ 0.38  \\
    LINKX  & 84.64 $\pm$ 1.13 & 73.19 $\pm$ 0.99 & 87.86  $\pm$ 0.77 \\ 
    PDE-GCN\textsubscript{M} & 88.60 $\pm$ N/A  & \textbf{78.48} $\pm$ N/A & 89.93 $\pm$ N/A \\
    NSD\textsuperscript{$\dagger$}  & 86.90 $\pm$ 1.13  & 76.70 $\pm$ 1.57   & 89.49 $\pm$ 0.40 \\
    GRAFF\textsuperscript{$\dagger$} & 88.01 $\pm$ 1.03 & 77.30 $\pm$ 1.85 & 90.04 $\pm$ 0.41 \\
    DMP\textsuperscript{$\dagger$}  & 86.52 $\pm$ N/A & 76.87 $\pm$ N/A & 89.27 $\pm$ N/A \\
    ACMII-GCN++ & 88.25 $\pm$ 0.96 & 77.12 $\pm$ 1.58 & 89.71 $\pm$ 0.48  \\
        GREAD\textsuperscript{$\dagger$} & 88.57  $\pm$ 0.66 & 77.60  $\pm$ 1.81 & 90.23  $\pm$ 0.55 \\
    \midrule
    RDGNN-I (Ours) & \textbf{89.91} $\pm$ \textbf{1.10} & 78.34 $\pm$  1.55& \textbf{90.37} $\pm$ \textbf{0.43} \\
    RDGNN-S (Ours)  & 89.53 $\pm$ 1.22 & 77.82 $\pm$ 1.48 & 90.11 $\pm$ 0.62\\

    \bottomrule
  \end{tabular}}
\end{table}

\begin{table*}[t]
\centering
  \caption{Node classification accuracy ($ \%$) on \emph{heterophilic} datasets. $\dagger$ denotes the maximal accuracy of several proposed variants.}
  \label{table:heterophilic_fully_std}
  \begin{center}
     \footnotesize
   \setlength{\tabcolsep}{1pt}
  \resizebox{1\linewidth}{!}{\begin{tabular}{ccccccc}
    \toprule
    Method & Squirrel & Film &  Cham. & Corn. & Texas & Wisc. \\
    Homophily & 0.22 & 0.22 & 0.23 & 0.30  & 0.11 & 0.21 \\
        \midrule
    GCN  & 23.96 $\pm$ 2.01 & 26.86  $\pm$  1.10 &  28.18  $\pm$  2.24&  52.70  $\pm$ 5.30  & 52.16  $\pm$ 5.16  & 48.92  $\pm$  3.06 \\
    GAT & 30.03  $\pm$  1.55 & 28.45  $\pm$ 0.89 & 42.93  $\pm$ 2.50 & 54.32  $\pm$ 5.05 & 58.38  $\pm$ 6.63 & 49.41  $\pm$ 4.09
    \\
    GCNII & 38.47  $\pm$ 1.58 & 32.87  $\pm$  1.30 &   60.61  $\pm$  3.04 & 74.86  $\pm$ 3.79 & 69.46  $\pm$ 3.83 & 74.12  $\pm$ 3.40 \\
    Geom-GCN\textsuperscript{$\dagger$}& 38.32  $\pm$ 0.92 & 31.63  $\pm$ 1.15 &  60.90  $\pm$ 2.81 & 60.81  $\pm$ 3.67 & 67.57  $\pm$ 2.72 & 64.12  $\pm$ 3.66 \\
    MixHop & 43.80  $\pm$  1.48 & 32.22  $\pm$ 2.34 & 60.50  $\pm$ 2.53 & 73.51  $\pm$ 6.34 & 77.84  $\pm$ 7.73 & 75.88  $\pm$ 4.90 \\ 
    GRAND & 40.05  $\pm$ 1.50 & 35.62  $\pm$ 1.01 &  54.67  $\pm$ 2.54 & 82.16  $\pm$ 7.09 & 75.68  $\pm$ 7.25 & 79.41  $\pm$ 3.64 \\ 
    PDE-GCN\textsubscript{M} & -- & -- &   66.01  $\pm$ N/A & 89.73 $\pm$ N/A  & 93.24 $\pm$ N/A  &  91.76 $\pm$ N/A \\
    NSD\textsuperscript{$\dagger$}  &  56.34  $\pm$  1.32 & 37.79  $\pm$ 1.15  & 68.68  $\pm$ 1.58  &  86.49  $\pm$ 4.71 & 85.95  $\pm$ 5.51 & 89.41  $\pm$  4.74 \\
    WRGAT & 48.85  $\pm$ 0.78 & 36.53  $\pm$ 0.77 & 65.24  $\pm$ 0.87 &  81.62  $\pm$ 3.90 & 83.62  $\pm$ 5.50 & 86.98  $\pm$ 3.78 \\
    MagNet &  --  & -- & --  & 84.30  $\pm$ 7.00 & 83.30  $\pm$ 6.10 & 85.70  $\pm$ 3.20 \\ 
    GGCN  & 55.17  $\pm$ 1.58 & 37.81  $\pm$ 1.56 &  71.14  $\pm$ 1.84 & 85.68  $\pm$ 6.63  & 84.86  $\pm$ 4.55  &  86.86  $\pm$ 3.29 \\
    H2GCN  & 36.48  $\pm$ 1.86 & 35.70  $\pm$ 1.00 & 60.11  $\pm$ 1.71 & 82.70  $\pm$ 5.28  & 84.86  $\pm$  7.23 &  87.65  $\pm$ 4.98 \\
    GraphCON\textsuperscript{$\dagger$} & -- & -- & -- & 84.30  $\pm$  4.80 & 85.40  $\pm$ 4.20  & 87.80  $\pm$ 3.30 \\  
    FAGCN & 42.59  $\pm$ 0.69 & 34.87  $\pm$ 1.35 & 55.22  $\pm$ 2.11 & 79.19  $\pm$ 5.87 & 82.43  $\pm$ 2.87 & 82.94  $\pm$ 1.58 \\
    GPRGNN & 31.61  $\pm$ 1.24 & 34.63  $\pm$ 1.22 & 46.58  $\pm$ 1.71 & 80.27  $\pm$ 8.11 & 78.38  $\pm$ 4.36 & 82.94  $\pm$ 4.21 \\
    ACMP-GCN & --  & -- & -- & 85.40  $\pm$  7.00 & 86.20  $\pm$ 3.00 & 86.10  $\pm$ 4.00  \\ 
    LINKX  & 61.81  $\pm$ 1.80 &  36.10  $\pm$ 1.55 & 68.42  $\pm$ 1.38 & 77.84  $\pm$ 5.81 & 
    74.60   $\pm$  8.37 & 75.49  $\pm$ 5.72 \\
    GRAFF\textsuperscript{$\dagger$} & 59.01  $\pm$ 1.31 & 37.11  $\pm$ 1.08 & 71.38  $\pm$ 1.47 & 84.05  $\pm$ 6.10 & 88.38  $\pm$ 4.53 & 88.83  $\pm$ 3.29\\ 
    DMP\textsuperscript{$\dagger$}  & 47.26 $\pm$ N/A & 35.72 $\pm$ N/A & 62.28 $\pm$ N/A & 89.19 $\pm$ N/A & 89.19 $\pm$ N/A & 92.16 $\pm$ N/A \\
G\textsuperscript{2}\textsuperscript{$\dagger$} & 64.26  $\pm$ 2.38 &  37.30  $\pm$ 1.01 & 71.40  $\pm$ 2.38 & 87.30  $\pm$ 4.84 &  87.57  $\pm$ 3.86  & 87.84  $\pm$ 3.49\\
    ACMII-GCN++ & \textbf{67.40  $\pm$ 2.21} & 37.09   $\pm$ 1.32 & 74.76  $\pm$ 2.2 & 86.49  $\pm$ 6.73 & 88.38  $\pm$ 3.43 & 88.43  $\pm$ 3.66 \\
    GREAD\textsuperscript{$\dagger$} & 59.22 $\pm$ 1.44 & 37.90 $\pm$ 1.17 & 71.38 $\pm$ 1.31 & 87.03 $\pm$ 4.95 & 89.73 $\pm$ 4.49 & 89.41 $\pm$ 3.30 \\
    \midrule
    RDGNN-I  (Ours)& 65.62 $\pm$ 2.33  & \textbf{38.69} $\pm$ \textbf{1.41} & \textbf{74.79} $\pm$ \textbf{2.14} & \textbf{92.72} $\pm$ \textbf{5.88} & 93.51 $\pm$ \
    5.93 &  \textbf{93.72} $\pm$  \textbf{4.59} \\
    RDGNN-S  (Ours)& 65.96 $\pm$ 2.11 & 38.55 $\pm$ 1.34 & 71.21 $\pm$ 2.08 & 92.43 $\pm$ 5.51 & \textbf{94.59} $\pm$ \textbf{5.97} &  92.94 $\pm$ 4.32 \\
    \bottomrule
  \end{tabular}}
\end{center}
\end{table*}

\begin{table*}[ht]
    \centering
    \caption{Test accuracy is provided (\%) for all datasets besides Twitch-DE that considers test ROC AUC$\ddagger$. Standard deviations are over 5 train/val/test splits. Not available results are indicated by --.}
    \label{tab:results}
    {\footnotesize
    \center{
    \begin{tabular}{ccccc}
    \toprule
    Method & Twitch-DE$\ddagger$ &  Deezer-Europe & Penn94 (FB100)  & arXiv-year  \\
    \midrule
     MLP & 69.20 $\std$ 0.62 &  66.55 $\std$ {0.72} & 73.61 $\std$ 0.40 & 36.70 $\std$ 0.21   \\    
     \hline
     L Prop (2 hop) & 72.27 $\std$ 0.78   & 56.96 $\std$ 0.26 & 74.13 $\std$ 0.46 &   $46.07\std{0.15}$  \\
     LINK & 72.42 $\std$ {0.57} & 57.71 $\std$ 0.36 & 80.79 $\std$ {0.49}  &  53.97 $\std$ {0.18}  \\    
     LINKX & -- & -- & 84.41 $\std$ 0.52 & 56.00 $\std$ {1.34} \\
     \hline
     SGC (2 hop) &  73.65 $\std$ {0.40} &  61.56 $\std$ {0.51} &  76.09 $\std$ 0.45 &   32.27 $\std$ {0.06}   \\    
     C\&S (2 hop) & 69.39 $\std$ 0.85 &  64.52 $\std$ 0.62 & 72.47 $\std$ 0.73 &    42.17 $\std$ 0.27 \\    
     \hline
     GCN &  74.07 $\std$ 0.68  & 62.23 $\std$ 0.53 &  82.47 $\std$ {0.27} &   46.02 $\std$ 0.26 \\
     GAT & 73.13 $\std$ 0.29 & 61.09 $\std$ {0.77} & 81.53 $\std$ {0.55}  &  46.05 $\std$ {0.51}  \\
     APPNP & 72.20 $\std$ {0.73} &  67.21 $\std$ {0.56} & 74.95 $\std$ {0.45} &   38.15 $\std$ {0.26}  \\
     \hline
     H2GCN & 72.67 $\std$ {0.65} &    67.22 $\std$ {0.90} & -- &    49.09 $\std$ {0.10}  \\
     GCNII & 72.38 $\std$ {0.31} & 66.42 $\std$ {0.56} & 82.92 $\std$ {0.59} & 47.21 $\std$ {0.28} \\
     MixHop & 73.23 $\std$ {0.99}   & 66.80 $\std$ {0.58} &  83.47 $\std${0.71} &      51.81 $\std$ {0.17}  \\
     GPR-GNN &  73.84 $\std$ {0.69}  &  66.90 $\std$ {0.50} &  84.59 $\std$ {0.29}   & 45.07 $\std$ {0.21}    \\
     ACMII-GCN++ & -- & \textbf{67.50} $\std$ \textbf{0.53} & 85.95 $\std$ {0.26} & -- \\
     \midrule
    RDGNN-I (Ours) & \textbf{74.31} $\std$ \textbf{0.74} & {67.43} $\std$ {0.46}  & \textbf{86.04} $\std$ \textbf{0.28} & \textbf{58.46} $\std$ \textbf{0.59}\\
     RDGNN-S (Ours) & {73.99} $\std$ {0.67} & 66.81 $\std$ {0.25} & 85.37 $\std$ {0.33} & 55.16 $\std$ {0.46} \\
    \bottomrule
    \end{tabular}}
    }
\end{table*}

\paragraph{Heterophilic data} While our RDGNN models offer competitive accuracy on homophilic datasets, the main benefit of a RD model, and as such our RDGNN, is its ability to model non-smooth patterns \textcolor{black}{that} often appear in heterophilic datasets by their definition \cite{Pei2020Geom-GCN:}, as visualized in Figure \ref{fig:heterophilic_data}. We therefore utilize 10 heterophilic datasets from various sources. In  \cref{table:heterophilic_fully_std} we report the average accuracy of our RDGNN with recent GNNs on the Squirrel, Film, and Chameleon from \cite{musae}, as well as the Cornell, Texas and Wisconsin datasets from \cite{Pei2020Geom-GCN:}, using the 10-splits from \cite{Pei2020Geom-GCN:} \textcolor{black}{with} train/validation/test split ratio of $48 \% / 32\% / 20\%$. In addition to the previously considered methods, here we also consider  FAGCN \cite{fagcn2021}, GraphCON \cite{rusch2022graph}, GPR-GNN \cite{chien2021adaptive}, GRAFF \cite{di2022graphGRAFF}, ACMP-GCN \cite{wang2022acmp}, and G\textsuperscript{2} \cite{rusch2022gradient}. We see that both our RDGNN-I and RDGNN-S variants offer high accuracy that is in line with recent state-of-the-art methods, and on most datasets RDGNN-I is the best performing model, besides on the Squirrel dataset where ACMI-GCN++ outperforms our RDGNN-S by 1.44\%.

\begin{figure}   
\centering
        \subfloat[Cornell]{
\includegraphics[width=0.32\textwidth]{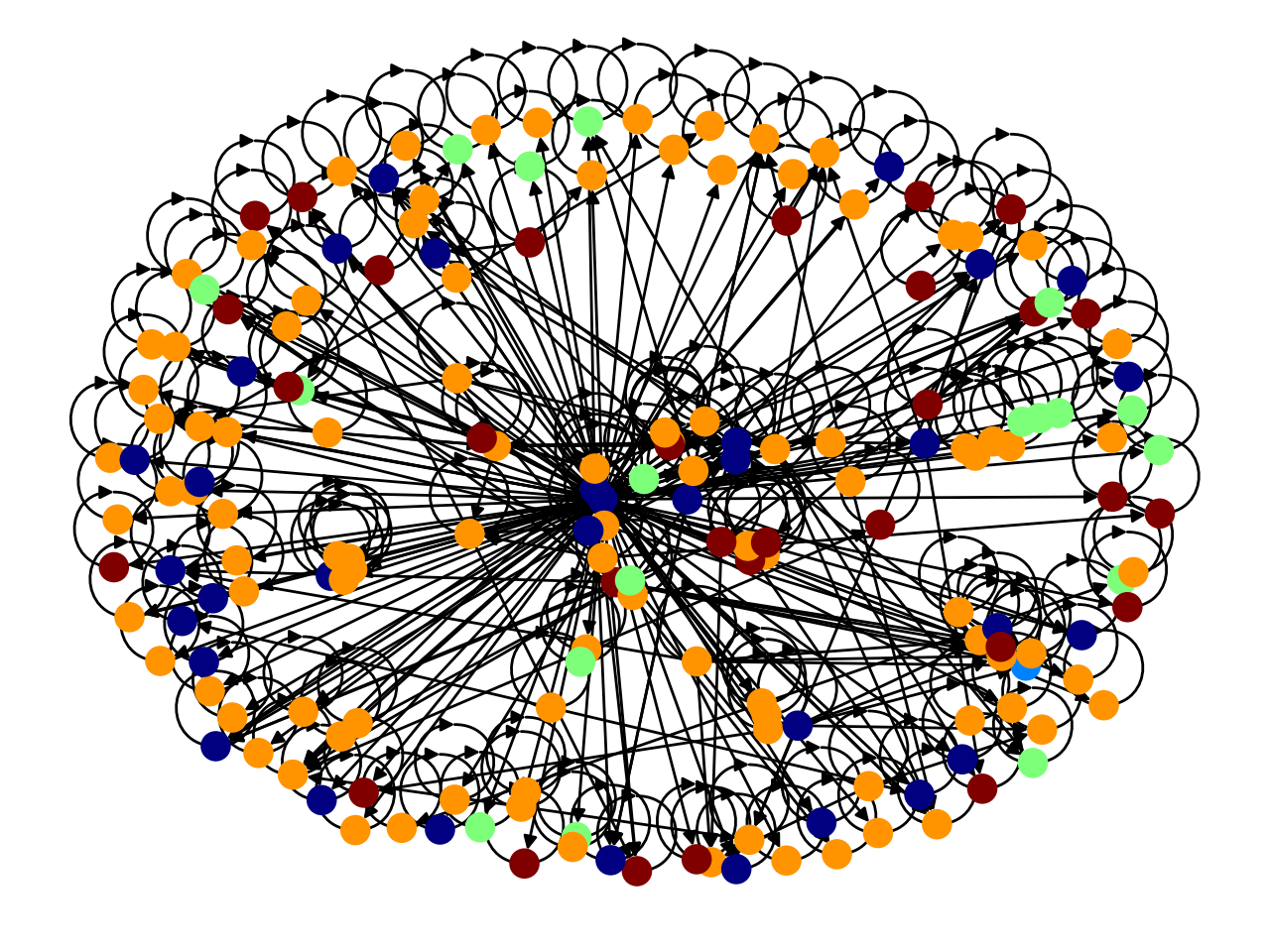}}
    \subfloat[Wisconsin]{    
\includegraphics[width=0.32\textwidth]{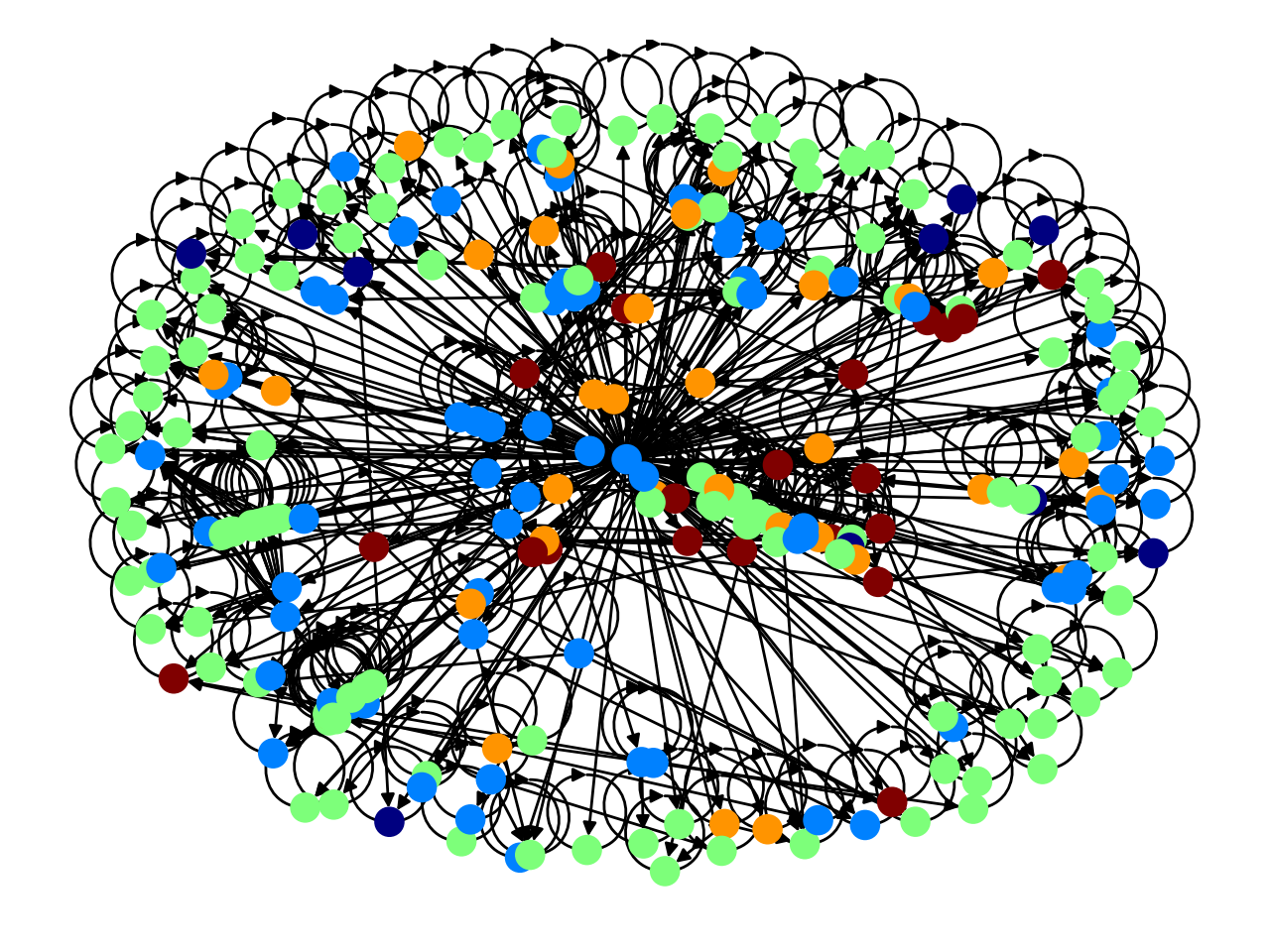}}
    \subfloat[Texas]{  
\includegraphics[width=0.32\textwidth]{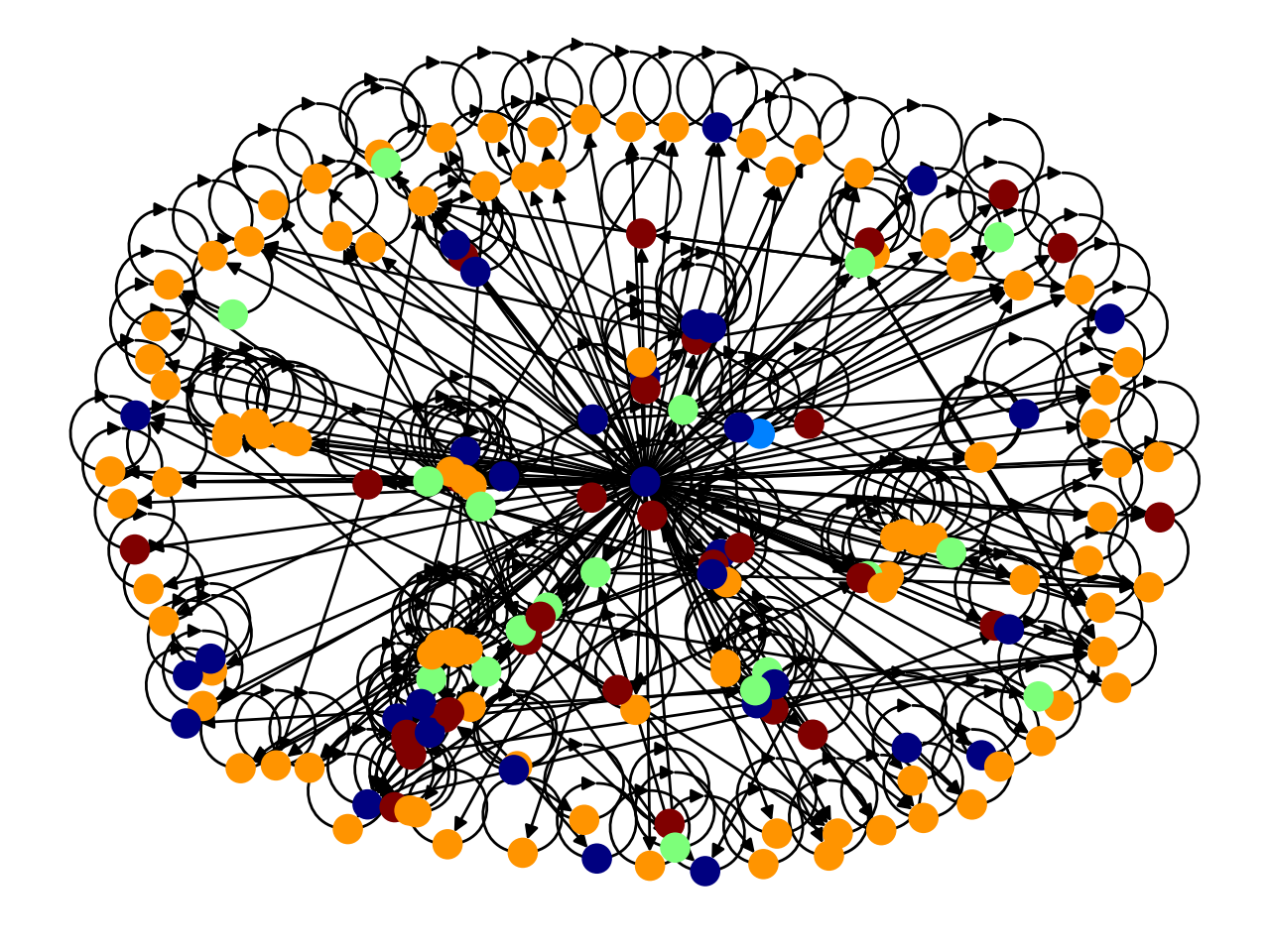}}
\caption{An example of the ground-truth labels of (a) Cornell (b) Wisconsin, and (c) Texas. Different colors indicate different labels. It is evident that non-smooth patterns exist in the data, which are suitable to be modeled with RD systems.}
\label{fig:heterophilic_data}
\end{figure}

Furthermore, it was recently identified by \cite{lim2021new, lim2021large} that the Cornell, Texas, and Wisconsin datasets suffer from the issues of being small and their tendency to yield high standard deviations in the obtained accuracy. Therefore, we also examine our RDGNN models with the datasets proposed by \cite{lim2021new, lim2021large} to facilitate the efficacy of our method further. Namely, in  \cref{tab:results} we provide the obtained accuracy of our RDGNN models and compare it with other recent methods on the Twitch-DE, deezer-europe, Penn94, and arXiv-year datasets from \cite{lim2021new, lim2021large}. In addition to previously mentioned works, here we also compare with SGC \cite{wu2019simplifying} and L prop \cite{peel2017graph} and LINK \cite{zheleva2009join}. Here, we see that our RDGNN-I achieves similar or better accuracy compared with other models. For example, our RDGNN-I obtains 87.04\% accuracy compared to the second best method ACMII-GCN++ with 85.95\%.

\subsection{Spatio-Temporal Node Regression}
\label{sec:temporalExperiments}
Classical RD models are widely utilized to predict and model spatio-temporal phenomena \cite{fiedler2003spatio}. We therefore now evaluate our RDGNN in such a scenario on several spatio-temporal node regression datasets.
Specifically, we harness the software package PyTorch-Geometric-Temporal by \cite{rozemberczki2021pytorch} that offers a graph machine learning pipeline with  \textcolor{black}{numerous} spatio-temporal graph datasets. In our experiments, we use the Chickenpox Hungary, PedalMe London, and Wikipedia Math datasets. We follow the same training and testing procedure from \cite{rozemberczki2021pytorch}, including the experimentation with both incremental and cumulative training modes \footnote{See \cite{rozemberczki2021pytorch} for details about training modes.}. We report the predictive performance of our RDGNN and other models, in terms of mean squared error (MSE), in  \cref{tab:predictive_performance}. We follow \cite{rozemberczki2021pytorch} and compare with several recent methods such as DCRNN \cite{li2018diffusion}, GConvGRU and GConvLSTM \cite{gconvlstm}, GC-LSTM \cite{gclstm}, DyGrAE \cite{dyggnn, dyngrae_1}, EGCN-H and EGCN-O \cite{evolvegcn}, A3T-GCN \cite{a3tgcn}, T-GCN \cite{tgcn}, MPNN LSTM \cite{panagopoulos2020transfer}, and AGCRN \cite{bai2020adaptive}. We observe that our RDGNN improves the considered baseline models. In particular, our time-embedded variant RDGNN-T  outperforms all the considered baseline methods by a large margin.

\begin{table*}[]
\centering
\caption{The predictive performance of spatio-temporal neural networks evaluated by average MSE of 10 experimental repetitions and standard deviations, calculated on 10\% forecasting horizons. We consider both incremental and cumulative backpropagation strategies.}\label{tab:predictive_performance}
{
\footnotesize
\resizebox{1\linewidth}{!}{
\begin{tabular}{ccccccc}
    \toprule
        \multirow{2}{*}{Method}
        & \multicolumn{2}{c}{{Chickenpox Hungary}}&  \multicolumn{2}{c}{{PedalMe London}} & \multicolumn{2}{c}{Wikipedia Math} \\
\cmidrule[0.4pt](lr{0.125em}){2-3}
\cmidrule[0.4pt](lr{0.125em}){4-5}
\cmidrule[0.4pt](lr{0.125em}){6-7}
       & Incremental     & Cumulative  & Incremental     & Cumulative     & Incremental     & Cumulative \\ \hline
{DCRNN} &    1.124 $\pm$ 0.015      &    1.123 $\pm$ 0.014       &       1.463$\pm$ 0.019     &        1.450 $\pm$ 0.024  &     
0.679 $\pm$ 0.020      &     0.803 $\pm$ 
 0.018        \\[0.1cm]
{GConvGRU} & 1.128 $\pm$ 0.011      &    1.132 $\pm$ 0.023           &      1.622 $\pm$ 0.032     &        1.944 $\pm$ 0.013  &     0.657 $\pm$ 0.015      &     0.837 $\pm$ 0.021        \\[0.1cm]
{GConvLSTM} & 1.121 $\pm$ 0.014      &    
 1.119 $\pm$ 0.022    &       1.442 $\pm$ 0.028     &        1.433 $\pm$ 0.020 &     0.777 $\pm$ 0.021      &     0.868 $\pm$ 0.018        \\[0.1cm]
{GC-LSTM} &    1.115 $\pm$ 0.014      &    1.116 $\pm$ 0.023       &      1.455 $\pm$ 0.023     &        1.468 $\pm$ 0.025  &     0.779 $\pm$  0.023      &     0.852 $\pm$ 0.016        \\[0.1cm]
{DyGrAE} & 1.120 $\pm$ 0.021      &    1.118 $\pm$ 0.015         &      1.455 $\pm$ 0.031     &        1.456 $\pm$ 0.019  &     0.773 $\pm$ 0.009      &     0.816 $\pm$ 0.016        \\[0.1cm]
{EGCN-H} & 1.113 $\pm$ 0.016      &    1.104 $\pm$ 0.024       &      $1.467\pm 0.026$     &        1.436 $\pm$ 0.017  &     0.775 $\pm$ 0.022      &     0.857 $\pm$ 0.022        \\[0.1cm]

{EGCN-O} &    1.124 $\pm$ 0.009      &    1.119 $\pm$ 0.020       &    1.491 $\pm$ 0.024     &        1.430 $\pm$ 0.023  &     
0.750 $\pm$ 0.014      &     0.823 $\pm$ 0.014        \\[0.1cm]
{A3T-GCN}& 1.114 $\pm$ 0.008      &    1.119 $\pm$ 0.018       &     1.469 $\pm$ 0.027     &        1.475 $\pm$ 0.029  &     0.781 $\pm$ 0.011      &     0.872 $\pm$ 0.017        \\[0.1cm]
{T-GCN} & 1.117 $\pm$ 0.011      &  1.111 $\pm$ 0.022  &    1.479 $\pm$ 0.012     &        1.481 $\pm$ 0.029  &     0.764 $\pm$ 0.011      &     0.846 $\pm$ 0.020        \\[0.1cm]
{MPNN LSTM} & 1.116 $\pm$ 0.023      &    1.129 $\pm$ 0.021       &       1.485 $\pm$ 0.028     &        1.458 $\pm$ 0.013  &     0.795 $\pm$ 0.010      &     0.905 $\pm$ 0.017        \\[0.1cm]
{AGCRN} & 1.120 $\pm$ 0.010      &    1.116 $\pm$ 0.017      &      1.469 $\pm$ 0.030     &        1.465 $\pm$ 0.026  &     0.788 $\pm$ 0.011      &     0.832 $\pm$ 0.020        \\[0.1cm]
\hline
{RDGNN-I} (Ours) & 1.110 $\pm$ 0.008 & 1.104 $\pm$ 0.004 &   1.271 $\pm$ 0.028  & 1.310 $\pm$ 0.019  & 0.621 $\pm$ 0.010   & 0.674  $\pm$ 0.022  \\
{RDGNN-S} (Ours) & 1.106 $\pm$ 0.008 & 1.103 $\pm$ 0.010  & 1.202 $\pm$ 0.032  & $1.278\pm 0.031$ & 0.612 $\pm$ 0.013 &  0.623 $\pm$ 0.014\\
{RDGNN-T} (Ours) & $\mathbf{1.100 \pm 0.005}$&  $\mathbf{1.097 \pm 0.003}$ &  $\mathbf{0.805 \pm 0.027}$  & $\mathbf{0.791 \pm  0.036}$  & $\mathbf{0.579 \pm 0.014}$ & $\mathbf{0.586 \pm 0.021}$ \\
\bottomrule
\end{tabular}
}}
\end{table*}

\subsection{Ablation Study}
\label{sec:ablation}
In this section, we aim to gain a deeper understanding of the contribution of each component of our RDGNN models.

\paragraph{RDGNN alleviates oversmoothing}
Our RDGNN models potentially prevent oversmoothing in two manners. The first is by using the learnable diffusion coefficients $\bfSigma$ that control the smoothness, and the second is due to the reaction function $f$, as described in Equation \eqref{eq:discDRa}. We therefore now study if our RDGNN models do not oversmooth in practice on real-world datasets. To this end, we train our RDGNN-I and RDGNN-\textcolor{black}{S} on the Cora, Chameleon, and Cornell datasets with a varying number of layers, from 2 to 64. We report the obtained accuracy in 
 \cref{fig:depth_study}, where it is evident that our RDGNN models do not oversmooth.
\textcolor{black}{
 Furthermore, it is important to note that while including the initial features $\bfU_0$ is beneficial in alleviating oversmoothing as shown in \cite{chen20simple}, our RDGNN can  alleviate oversmoothing also without including $\bfU_0$, as theoretically discussed in \ref{sec:stabilityAnalysis}. We also empirically verify this theoretical understanding as shown in Figure \ref{fig:depth_study_u0}. We observe that the absolute accuracy is slightly lower than the case where $\bfU_0$ is included in the learning. The improved performance when considering $\bfU_0$ can be attributed to the increased complexity of the network. Notably, it is evident that an RDGNN that does not include $\bfU_0$ in the reaction term, denoted by RDGNN\textsubscript{no$\bfU_0$}, does not oversmooth.}

\begin{figure}[h]
    \centering
    \centering
    \begin{tikzpicture}
      \begin{axis}[
          width=0.7\linewidth, 
          height=0.42\linewidth,
          grid=major,
          grid style={dashed,gray!30},
          xlabel=Layer,
          ylabel=Accuracy (\%),
          ylabel near ticks,
          legend style={at={(0.65,0.25)},anchor=north,scale=0.8, draw=none, cells={anchor=west}, font=\tiny, fill=none},
          legend columns=3,
          xtick={1,2,4,8,12,16,32, 64},
          xticklabels = {1,2,4,8,12,16,32, 64},
          yticklabel style={
            /pgf/number format/fixed,
            /pgf/number format/precision=3
          },
          ticklabel style = {font=\tiny},
          scaled y ticks=false,
          every axis plot post/.style={thick},
          ymin = 60
        ]
        \addplot[blue, mark=oplus*, forget plot]
        table[x=nlayer
,y=cora_ind,col sep=comma] {data/depth_study.csv};
        \addplot[green, mark=oplus*, forget plot]
        table[x=nlayer
,y=chameleon_ind,col sep=comma] {data/depth_study.csv};
        \addplot[red, mark=oplus*, forget plot]
        table[x=nlayer
,y=texas_ind,col sep=comma] 
{data/depth_study.csv};
        \addplot[blue, style=dashed, mark=oplus*, forget plot]
        table[x=nlayer
,y=cora_share,col sep=comma] {data/depth_study.csv};
        \addplot[green, style=dashed, mark=oplus*, forget plot]
        table[x=nlayer
,y=chameleon_share,col sep=comma] {data/depth_study.csv};
        \addplot[red, style=dashed, mark=oplus*, forget plot]
        table[x=nlayer
,y=texas_share,col sep=comma] {data/depth_study.csv};

    \addplot[blue, draw=none] coordinates {(1,1)};
    \addplot[green, draw=none] coordinates {(1,1)};
    \addplot[red, draw=none] coordinates {(1,1)};

    \addplot[gray, draw=none] coordinates {(1,1)};
    \addplot[gray, style=dashed, draw=none] coordinates {(1,1)};

    \legend{Cora, Chameleon, Texas, {RDGNN-I}, {RDGNN-S}}
        \end{axis}
    \end{tikzpicture}
\caption{The accuracy (\%) on homophilic and heterophilic datasets vs. model depth. Our RDGNN models do not suffer from oversmoothing.}
\label{fig:depth_study}
\end{figure}
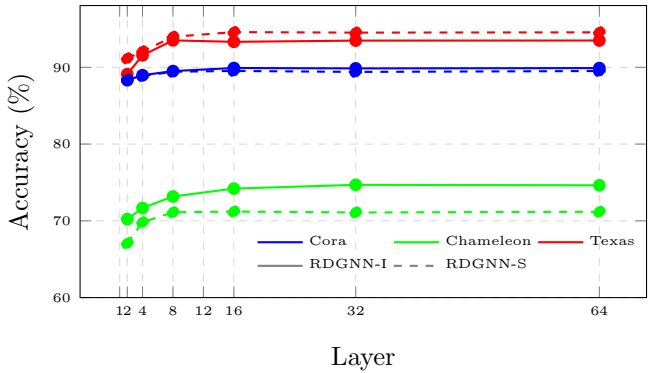

\begin{figure}[h]
    \centering
    \centering
    \begin{tikzpicture}
      \begin{axis}[
          width=0.7\linewidth, 
          height=0.42\linewidth,
          grid=major,
          grid style={dashed,gray!30},
          xlabel=Layer,
          ylabel=Accuracy (\%),
          ylabel near ticks,
          legend style={at={(0.65,0.25)},anchor=north,scale=0.8, draw=none, cells={anchor=west}, font=\tiny, fill=none},
          legend columns=2,
          xtick={1,2,4,8,12,16,32, 64},
          xticklabels = {1,2,4,8,12,16,32, 64},
          yticklabel style={
            /pgf/number format/fixed,
            /pgf/number format/precision=3
          },
          ticklabel style = {font=\tiny},
          scaled y ticks=false,
          every axis plot post/.style={thick},
          ymin = 60
        ]
        \addplot[blue, mark=oplus*, forget plot]
        table[x=nlayer
,y=cora_ind,col sep=comma] {data/depth_study_u0.csv};
        \addplot[green, mark=oplus*, forget plot]
        table[x=nlayer
,y=chameleon_ind,col sep=comma] {data/depth_study_u0.csv};

        \addplot[blue, style=dashed, mark=oplus*, forget plot]
        table[x=nlayer
,y=cora_ndrp,col sep=comma] {data/depth_study_u0.csv};
        \addplot[green, style=dashed, mark=oplus*, forget plot]
        table[x=nlayer
,y=chameleon_ndrp,col sep=comma] {data/depth_study_u0.csv};

    \addplot[blue, draw=none] coordinates {(1,1)};
    \addplot[green, draw=none] coordinates {(1,1)};

    \addplot[gray, draw=none] coordinates {(1,1)};
    \addplot[gray, style=dashed, draw=none] coordinates {(1,1)};

    \legend{Cora, Chameleon, {RDGNN}, {RDGNN\textsubscript{no$\bfU_0$}}}
        \end{axis}
    \end{tikzpicture}
\caption{\textcolor{black}{RDGNN alleviates oversmoothing regardless of the inclusion of $\bfU_0$ in the reaction term. The base RDGNN here is RDGNN-I.}}
\label{fig:depth_study_u0}
\end{figure}
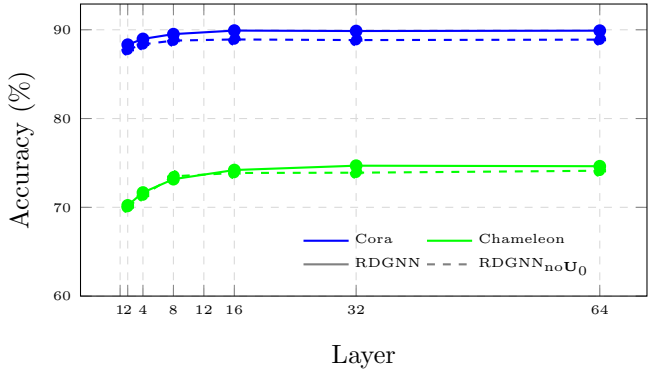

\paragraph{Explicit vs. Implicit-Explicit Integration}
In \cref{sec:timeDiscretization} we discuss two time-integration methods. The first method is the forward Euler integration scheme, and the second is the \textcolor{black}{numerically} more stable implicit-explicit (IMEX) \cite{ascher2008numerical}. It is therefore interesting to understand if the IMEX scheme performs better on real-world datasets. According to our results, provided in  \cref{table:integration_scheme}, we see that the IMEX approach offers consistently better performance compared to the explicit method on 3 different datasets, both homophilic and heterophilic. Thus, we use the IMEX discretization scheme throughout the rest of our experiments.

\begin{table}[]
  \caption{Impact of Time Integration scheme.}
  \label{table:integration_scheme}   
  \center{
  \footnotesize
  \begin{tabular}{ccccc}
    \toprule
    Integration & Method & Cora & Cham. & Texas \\
   scheme & Homophily & 0.81 & 0.22 & 0.11 \\
    \midrule
    
    Explicit & RDGNN-I  & 88.89 & 68.91 & 92.80 \\
    & RDGNN-S & 88.67 & 67.22 & 94.04 \\
    \midrule
        IMEX &RDGNN-I & 89.91 & 74.69  &  93.51\\
    &RDGNN-S  & 89.53 &  71.21 & 94.59\\

    \bottomrule
  \end{tabular}}
\end{table}

\paragraph{The influence of the reaction function $f$}
In \cref{sec:ReactionAndDiffusion_Funcs} we discuss  our reaction term, starting from a simple MLP, called an \emph{additive} layer, that considers only the features of the current layer $\bfU_{l}$. See \cref{appendix:architectures} for more details. Furthermore, we  consider the input features embedding $\bfU_{0}$ using an additional layer as described in  \cref{eq:simpleAdditiveReaction}. Lastly, we suggest, based on typical RD systems \cite{fiedler2003spatio}, to use a multiplicative layer as described in  \cref{eq:multiplicativeLayer}. We report the obtained accuracy using each of the proposed terms individually and compare it to our full reaction term from \cref{eq:totalReaction} (All), in \cref{table:reactionFunction}. We see that considering the initial features \textcolor{black}{$\bfU_0$ in the reaction term is beneficial for both homophilic and heterophilic datasets, compared to a simple additive layer that can be achieved by setting $\bfK_{\bfU_0}=0$ in Equation \eqref{eq:simpleAdditiveReaction}, effectively not considering $\bfU_0$ as part of the reaction function.}  In addition, we find that the multiplicative reaction term is beneficial for heterophilic datasets. For example, our RDGNN-I with a multiplicative layer offers a 1.7\% improvement compared to the simple additive layer on the Chameleon dataset.

\begin{table}[]
  \caption{Impact of the reaction function.}
  \label{table:reactionFunction}   
  \center{
  \footnotesize
  \begin{tabular}{ccccc}
    \toprule
    Reaction & Method & Cora & Cham. & Texas \\
    function & Homophily & 0.81 & 0.22 & 0.11 \\
    \midrule
    
    Additive & RDGNN-I  & 88.81 & 72.09 & 91.30\\
    &RDGNN-S & 88.64 & 68.91 & 91.94 \\
    \midrule
        Additive with $\bfU_0$ & RDGNN-I  & 89.55 & 72.88 & 92.50 \\
    &RDGNN-S  & 89.38 & 69.20 & 93.34\\
    \midrule
    Multiplicative  & RDGNN-I  &  89.34 & 73.79 & 93.27 \\
    &RDGNN-S  & 89.17 & 70.83 & 93.76 \\
    \midrule
    All  & RDGNN-I  & 89.91 & 74.69 & 93.51  \\
    &RDGNN-S  &  89.53 & 71.21 & 94.59\\
    \bottomrule
  \end{tabular}}
\end{table}
\paragraph{\textcolor{black}{Influence of the Dropout layer}} \textcolor{black}{A common practice in neural networks, and specifically in node-classification in GNNs \cite{velickovic2018graph,rusch2022gradient} is the utilization of the dropout layer \cite{srivastava2014dropout} to improve performance and generalization. We now show the impact of applying the dropout layer to our RDGNN, as prescribed in Appendix \ref{appendix:architectures}. The results, shown in Table \ref{table:dropout}, indicate that similar to previous works, adding dropout to the optimization stage of RDGNN improves its accuracy.}

\begin{table}[]
  \caption{\textcolor{black}{Impact of using Dropout in RDGNN.}}
  \label{table:dropout}   
  \center{
  \footnotesize
  \begin{tabular}{ccccc}
    \toprule
     \multirow{2}{*}{Dropout} & Method & Cora & Cham. & Texas \\
     & Homophily & 0.81 & 0.22 & 0.11 \\
    \midrule
    
    \xmark  & RDGNN-I  & 87.90 & 72.98 & 91.22  \\
    &RDGNN-S  &  87.64 & 70.19 & 92.03\\
    \midrule

    \cmark  & RDGNN-I  & 89.91 & 74.69 & 93.51  \\
    &RDGNN-S  &  89.53 & 71.21 & 94.59\\
    \bottomrule
  \end{tabular}}
\end{table}

\section{Conclusion}
\label{sec:conclusion}
In this paper we proposed RDGNN - a new family of reaction diffusion graph neural networks that offers a perspective on GNNs from a locally unstable dynamical system, taking inspiration from Turing instabilities. The RD equation is discretized with an implicit\textcolor{black}{-explicit} integration scheme to yield RDGNN. We show analysis for the behavior of the time-continuous RD equation, explaining how local instabilities can occur. Because of this property, our GNN is most suitable for datasets that require the generation of non-smooth patterns, e.g., heterophilic datasets, but can also be trained to yield a diffusive process if the dataset requires it (e.g, for homophilic datasets). Through an extensive set of experiments, on  static, temporal, homophilic, and heterophilic datasets, we demonstrate the advantage of our RDGNN over existing models.

\appendix

\section{Architectures}
\label{appendix:architectures}
The general architecture of our RDGNN variants is given in \cref{table:generalArch}.
We initialize all $1\times1$ convolutions, and time embedding parameters (when using RDGNN-T) using  Glorot \cite{glorot2010understanding} initialization. We initialize the parameter $\hat{\bfSigma}$ from \cref{eq:diffusion} such that the diffusion coefficients $\bfSigma$ start from ones.

\paragraph{The Reaction and Diffusion Functions}
In \cref{sec:ReactionAndDiffusion_Funcs} we discuss the additive and multiplicative reaction functions used in our experiments, as well as the parameterization of the diffusion coefficients $\bfSigma(t)$. 
The reaction function $f$ in \cref{eq:simpleAdditiveReaction}, and the diffusion coefficient parameterization in \cref{eq:diffusion} are written in a most generic manner, as a function of time $t$. For clarity, it is important to note that when utilizing our time-\emph{independent} variants RDGNN-I or RDGNN-S, the time embedding is not learnt, leading to the following corresponding time-independent reaction term
\begin{equation}
    \label{eq:noTimeEmbeddingReaction}
            f_{\rm add}(\bfU, \bfU_{0}) = \sigma \left(\bfU \bfK_{U}  + \bfU_{0} \bfK_{U_{0}}  \right),
\end{equation}
and diffusion term
\begin{equation}
    \label{eq:noTimeEmbeddingDiffusion}
    \bfSigma = \exp\left(-{\rm{ReLU}}(\hat{\bfSigma}) \right).
\end{equation}
When utilizing our time-embedded RDGNN-T, we use  \cref{eq:simpleAdditiveReaction} and \cref{eq:diffusion}. Furthermore, in our ablation study in \cref{sec:ablation}, first row in \cref{table:reactionFunction}, we experiment with a reaction function that does not take the initial node embedding $\bfU_{0}$ into consideration, as follows:
\begin{equation}
    \label{eq:noTimeEmbeddingReactionnoU0}
            f_{\rm add}(\bfU) = \sigma \left(\bfU \bfK_{U} \right).
\end{equation}

\begin{table}[H]
  \caption{The general RDGNN architecture.}
  \label{table:generalArch}
  \begin{center}
  \begin{tabular}{lcc}
  \toprule
    Input size & Layer  &  Output size \\
    \midrule
    $n \times c_{in}$ & \textcolor{black}{Dropout(p)} & $n \times c_{in}$ \\
    $n \times c_{in}$ & $1\times1$ Convolution & $n \times c$ \\
    $n \times c$ & ReLU & $n \times c$ \\    
    $n \times c$ & $L \times $ RDGNN layers & $n \times c$ \\
    $n \times c$ & Dropout(p) & $n \times c$ \\
    $n \times c$ & $1\times1$ Convolution & $n \times c_{out}$ \\
    \bottomrule
  \end{tabular}
\end{center}
\end{table}

\section{Hyperparameters}
\label{appendix:hyperparams}
All hyperparameters were determined by grid search, and the ranges and sampling mechanism distributions are provided in Table \ref{tab:hyperparams}. 

\begin{table}[h]
\centering
\caption{Hyperparameter ranges}
{
\label{tab:hyperparams}
\begin{tabular}{ccc}
\toprule
Hyperparameter   & Range & Uniform Dist. \\
\midrule
     input/output embedding learning rate &  [1e-4, 1e-1] & log uniform  \\
    diffusion learning rate & [1e-4, 1e-1] & log uniform   \\
      reaction learning rate &  [1e-4, 1e-1] & log uniform  \\
       input/output embedding weight decay &  [0, 1e-2] & uniform \\
    diffusion weight decay & [0, 1e-2] & uniform  \\
      reaction weight decay & [0, 1e-2] & uniform  \\
      input/output dropout &  [0, 0.9] & uniform \\
      hidden layer dropout & [0, 0.9] & uniform \\
      use BatchNorm & \{ yes / no \} & discrete uniform \\
      step size h &  [1e-3, 1] & uniform \\
      layers & \{ 2,4,8,16,32,64 \} & discrete uniform \\
      channels &  \{ 8,16,32,64,128,256 \} & discrete uniform  \\
 \bottomrule
\end{tabular}
}
\end{table}

\section{Datasets}
\label{appendix:datasets}
We report the statistics of the datasets used in our experiments in \cref{table:datasets} and \cref{tab:desc_discrete} for the node classification, and spatio-temporal node regression datasets, respectively.
All datasets are publicly available, and appropriate reference to the data sources is provided in the main text.

\begin{table}[H]
  \caption{Node classification datasets statistics.}
  \label{table:datasets}
  \begin{center}
  \begin{tabular}{lccccc}
  \toprule
    Dataset & Classes & Nodes & Edges & Features & Homophily \\
    \midrule
    Cora & 7 & 2,708 & 5,429 & 1,433 & 0.81\\
    Citeseer & 6 & 3,327  & 4,732 & 3,703 & 0.80\\
    Pubmed & 3 & 19,717 & 44,338 & 500 & 0.74 \\
    Chameleon & 5 & 2,277 &  36,101 & 2,325 & 0.23\\
    Film & 5 & 7,600 & 33,544 & 932 & 0.22  \\
    Squirrel & 5 & 5,201 & 198,493 &  2,089 & 0.22 \\
    Cornell & 5 & 183 & 295 & 1,703 & 0.30\\
    Texas & 5 & 183 & 309 & 1,703 & 0.11 \\
    Wisconsin & 5 & 251 & 499 & 1,703 & 0.21 \\
    Twitch-DE & 2 & 9,498 & 76,569 & 2,545 & 0.63 \\
    Deezer-Europe & 2 &  28,281 &  92,752 & 31,241  & 0.52  \\
    Penn94 (FB100) &  2 &  41,554 &  1,362,229 & 5 & 0.47  \\
    arXiv-year & 5 & 169,343 &  1,166,243 & 128 & 0.22 \\
    \bottomrule
  \end{tabular}
\end{center}
\end{table}

\begin{table}[t]
\centering
\caption{Attributes of the spatio-temporal datasets used in  \cref{sec:temporalExperiments} and information about the number of time periods ($T$) and spatial units ($|\mathcal{V}|$).}\label{tab:desc_discrete}
{
\begin{tabular}{cccccc}
\toprule
Dataset & Signal & Graph & Frequency & $T$ & $|\mathcal{V}|$ \\
\midrule
    Chickenpox Hungary & Temporal&Static & Weekly & 522 & 20 \\
     Pedal Me Deliveries & Temporal & Static & Weekly & 36 & 15 \\
 Wikipedia Math & Temporal&Static&Daily & 731 & 1,068 \\
 \bottomrule
\end{tabular}
}
\end{table}

\section{Complexity \textcolor{black}{ and Runtimes}}
\label{app:complexity}
\paragraph{Number of parameters}
Typical neural networks, including GNNs usually employ new learnable weights per layer, i.e., the number of parameters increases as more layers are added. In our RDGNN we offer both formulations. Our RDGNN-I utilizes individual weights per layer, while our RDGNN-S and RDGNN-T share parameters across layers. In our experimental results on node classification in \cref{sec:nodeExperiments} we see that both RDGNN-I and RDGNN-S offer high accuracy that is in line with recent state-of-the-art methods. In our spatio-temporal experiments in \cref{sec:temporalExperiments}, we found that sharing parameters with time embedding, i.e., our RDGNN-T offers the highest results among our proposed RDGNN variants.

\paragraph{Implicit-explicit solution}
The system presented here requires the solution of the linear system at each step. As previously discussed this can be achieved by direct methods (Cholesky) for small-scale problems or, by using the CG method for large-scale problems.
While the backward function is implemented in
different software packages for the inverse
matrix, it is not so for the CG algorithm which may require the automatic differentiation to trace through the CG iterations. This can be avoidable by using implicit differentiation, which is the workhorse of implicit methods (see \cite{haberBook2014} for detailed derivation). In the context of deep learning, implicit differentiation was used for
implicit neural networks \cite{gu2020implicit}.
The basic idea is to use implicit differentiation
of the equation 
$${\rm vec}(\bfV) = ((\bfI + h \bfSigma \otimes \hat\bfL ) \rm{vec}(\bfU)) $$
with respect to  ${\bfSigma}$  and thus avoid the potentially expensive tracing of the CG iterations if many are needed.
{\paragraph{\textcolor{black}{Runtimes}} \textcolor{black}{In addition to the complexity analysis above, we  provide the measured runtimes in Table \ref{tab:runtimes}. Learning the reaction term requires an increased computational effort compared to baseline models like GCN. However, it also improves the task accuracy. For convenience, in Table \ref{tab:runtimes}, in addition to the runtimes, we also report the obtained task metric. Importantly, we find that the improved accuracy offered by RDGNN with 64 channels and 16 layers (where RDGNN reaches to accuracy plateau with respect to the model's depth) - is not simply obtained due to the increased costs (i.e., more layers or parameters). This is evident, as we see that enlarging GCN and GAT from standard 2 layers and 64 channels, to 2 layers and 256 channels (wide), or 64 layers and 64 channels (deep) does not yield similar improvements. We measured the runtimes using an Nvidia-RTX3090 with 24GB of memory, which is the same GPU used to conduct our experiments.}

\begin{table*}[t]
\footnotesize
   \setlength{\tabcolsep}{4pt}
  \caption{\textcolor{black}{Training and inference GPU runtimes (milliseconds), number of parameters (thousands), and node classification accuracy (\%) on Cora.}}
  \label{tab:runtimes}
  \begin{center}
  \begin{tabular}{lccccccccccc}
  \toprule
   \multirow{2}{*}{Metric}  & \multirow{2}{*}{GCN} & \multirow{2}{*}{GAT} & GCN  & GAT& GCN & GAT    & \multirow{2}{*}{RDGNN-{I}}  & \multirow{2}{*}{RDGNN-{S}}  \\
   & & & (wide) & (wide) & (deep) & (deep) &  \\
    \midrule
    Training time & 7.71 & 14.59  & 14.32 & 36.63 & 95.11 & 184.51 & 121.54 & 121.31 \\
    Inference time  & 1.75 & 2.98   &  2.86 & 7.57 & 12.93 &  38.96 & 33.87 & 33.11 \\
    Parameters & 104 & 105 & 565 & 567 & 358 & 360 & 436 & 190 \\
    Accuracy &  85.77 & 83.13  & 85.18 & 83.37 & 38.62 & 33.40  & 89.91 & 89.53\\
    \bottomrule
  \end{tabular}
  \end{center}
\end{table*}

\bibliographystyle{siamplain}
\bibliography{biblio}
\end{document}